%% file: main.tex
\documentclass[12pt]{article}

\usepackage[letterpaper,top=2cm,bottom=2cm,left=3cm,right=3cm,marginparwidth=1.75cm]{geometry}

\usepackage{amsmath}
\usepackage{amssymb}
\usepackage{amsthm}
\newtheorem{proposition}{Proposition}
\newtheorem{theorem}{Theorem}
\newtheorem{corollary}{Corollary}[theorem]

\newtheorem{assumption}{Assumption}

\theoremstyle{definition}

\newtheorem{remark}{Remark}
\usepackage{graphicx}
\usepackage[colorlinks=true, allcolors=blue]{hyperref}
\usepackage[labelfont=bf,justification=justified,format=plain]{caption}
\usepackage{subcaption}
\usepackage{float}
\usepackage{harvard}
\usepackage{cleveref}
\crefname{lemma}{Lemma}{Lemmas}
\crefname{proposition}{Proposition}{propositions}
\crefname{theorem}{Theorem}{Theorems}
\crefname{assumption}{Assumption}{Assumptions}
\crefname{algorithm}{Algorithm}{Algorithms}
\crefname{section}{Section}{Sections}
\crefname{corollary}{Corollary}{corollaries}
\crefname{remark}{Remark}{remarks}
\crefname{appendix}{Appendix}{appendices}
\usepackage{algorithm}
\usepackage{algpseudocode}
\usepackage{crossreftools}
\let\ORGhypersetup\hypersetup
\protected\def\hypersetup{\ORGhypersetup}
\DeclareMathOperator{\argmin}{arg\,min}
\title{\textbf{Lassoed Forests: Random Forests with\\ Adaptive Lasso Post-selection}}
\author{Jing Shang$^{1}$ \; James Bannon$^{3}$ \; Benjamin Haibe-Kains $^{3,4}$\; Robert Tibshirani$^{1,2}$}
\date{\normalsize
$^1$Department of Statistics \quad $^2$Department of Biomedical Data Science \\
Stanford University\\
$^3$Princess Margaret Cancer Centre \quad $^4$Department of Medical Biophysics\\
University of Toronto}

\begin{document}
\maketitle

\begin{abstract}
Random forests are a statistical learning technique that use bootstrap aggregation to average high-variance and low-bias trees. Improvements to random forests, such as applying Lasso regression to the tree predictions, have been proposed in order to reduce model bias. However, these changes can sometimes degrade performance (e.g., an increase in mean squared error). In this paper, we show in theory that the relative performance of these two methods, standard and Lasso-weighted random forests, depends on the signal-to-noise ratio. We further propose a unified framework to combine random forests and Lasso selection by applying adaptive weighting and show mathematically that it can strictly outperform the other two methods. We compare the three methods through simulation, including bias-variance decomposition, error estimates evaluation, and variable importance analysis. We also show the versatility of our method by applications to a variety of real-world datasets.
\end{abstract}

\section{Introduction}
Tree-based methods are a family of non-parametric approaches in supervised learning. Random forests use a form of bootstrap aggregation, or \textit{bagging}, to combine a large collection of trees and produce a final prediction. In regression problems, it gives the same weight to each tree and computes the average out-of-bag prediction. In classification problems, it assigns class labels by majority vote. However, since a single-tree model is known to have high variance, a large number of trees need to be trained and aggregated in order to reduce variance \cite**{hastie2009random}. This can lead to redundant trees, as the bootstrap procedure may select similar sets of samples to train different trees. Moreover, increasing  the number of trees does not reduce the bias.

Post-selection boosting random forests, proposed by \citeasnoun{wang2021improving}, is an attempt to reduce bias by applying Lasso regression \cite{tibshirani1996regression} on the predictions from each individual tree. The method returns a sparser forest with fewer trees, as well as different weights assigned to each individual tree. However, our numerical experiments show that the post-selection method only leads to an improvement in prediction error under certain signal-to-noise ratio conditions. In other words, the post-selection method can hurt performance in some cases.

Based on this observation, this paper proposes a unified framework to combine random forest and Lasso selection in an adaptive manner. For simplicity of notation, we refer to the classic random forest \cite{breiman2001random} that takes simple average over out-of-bag predictions as \textit{vanilla forest}, the post-selection boosting random forest as \textit{post-selection forest}, and our proposed method as \textit{Lassoed forest}. The Lassoed forest treats the predictions from vanilla forest as an offset, and applies adaptive weights on the Lasso selection as well as the offset term. Our method estimates the optimal weights by assessing the out-of-bag and cross-validation error estimates. We note that both the vanilla forest and post-selection forest are special cases of our proposed method, with all weights on either the Lasso selection or the offset term. Hence, our method is a direct generalization and a potential improvement as it uses adaptive weights based on data.

\subsection{Motivating Examples} \label{subsec:examples}

\textbf{A. California House Prices Prediction.}
First, we show an example where the post-selection forest outperforms  the vanilla forest. We use the California housing dataset from \citeasnoun{pace1997sparse}. It contains 9 predictors with 20,640 observations, including median income, housing median age, total rooms, etc. It forms a regression problem, with log-transformed median house value being the response variable.

\autoref{fig:cahousing} shows that in this regression task, the post-selection forest outperforms vanilla forest, and its advantage increases with the number of trees constructed. Our proposed Lassoed forest  appropriately gives more weight to the post-selection forest. An intuitive explanation of this improvement is that the strong predictors in this problem (e.g., housing age has high positive correlation with price) allow the Lasso to deliver extra bias reduction without introducing too much variance.
\begin{figure}[H]
    \centering
    \includegraphics[width=0.75\linewidth]{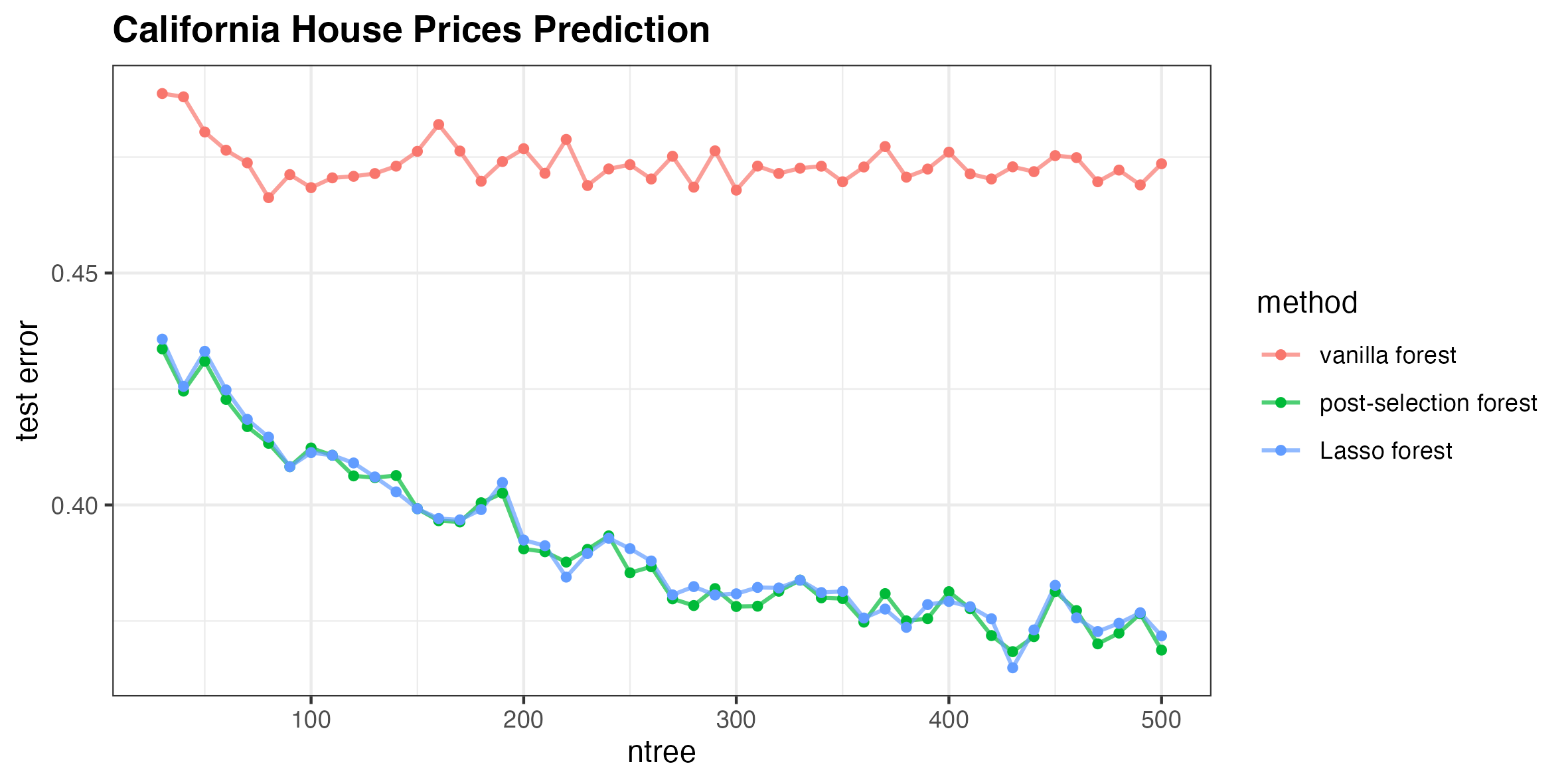}
    \caption{\em{Test Error for California House Prices Prediction}}
    \label{fig:cahousing}
\end{figure}

\noindent\textbf{B. Spam Classification.}
Next, we show that the post-selection forest does not guarantee an improvement, which motivates our proposed Lassoed forest. We apply all three methods on the Spambase dataset from \citeasnoun**{hopkins1999spambase}. It contains 57 words with high frequency extracted from 4,601 email messages, each labeled as spam or non-spam. The task is to train a classifier to filter out spam emails.

The post-selection forest can be readily extended to such classification problem in the following way: first we fit a random forest to predict class probability with a fixed number of trees, then we run a generalized linear model with the binomial family and Lasso regularization. \autoref{fig:spam} shows that the post-selection forest performs worse than vanilla forest in this classification task, though the gap shrinks as the number of trees increases. Again, our proposed Lassoed forest puts virtually all of the weight on the vanilla forest. A heuristic explanation is that when each tree is built on weak predictors (e.g., a single word could have low correlation with the class label), the Lasso reduces very little bias, while adding variance.

\begin{figure}[H]
    \centering
    \includegraphics[width=0.75\linewidth]{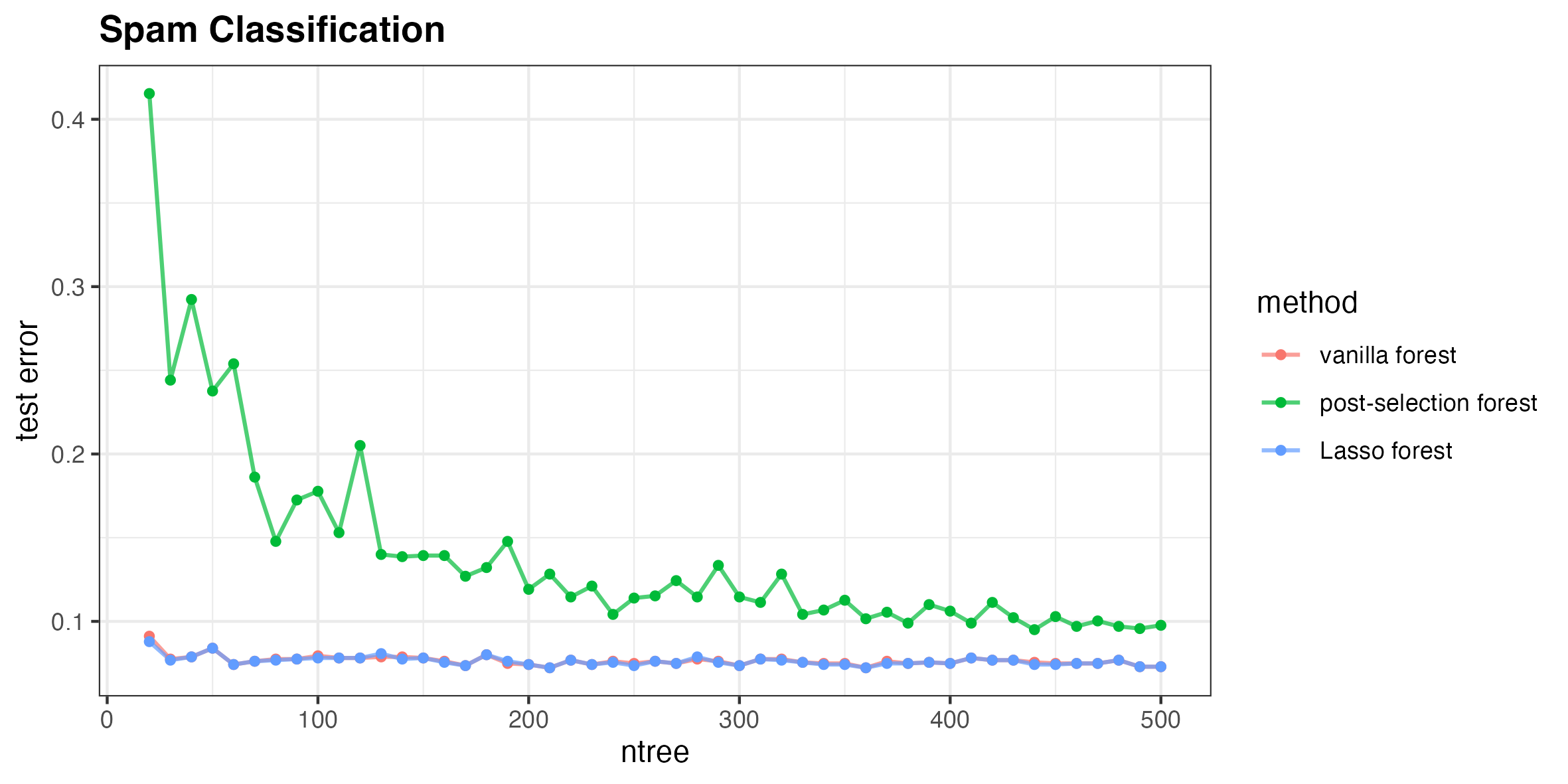}
    \caption{\em{Test Error for Spam Classification}}
    \label{fig:spam}
\end{figure}

\subsection{Outline of Paper}
The rest of this paper is structured in the following way: first, we formally introduce the formulation of the Lassoed forest in \cref{sec:methodology}. As mentioned above, it is motivated by the observation of different performances from vanilla forest and post-selection forest in different cases. In \cref{sec:theory}, we provide a theoretical analysis of our proposed method, including the importance of the signal-to-noise ratio, a bias-variance analysis, and theoretical guarantees of improvement in prediction error. Next, we show empirical evidence of prediction error improvement, through simulation studies with various forms of data generating processes in \cref{sec:simulation}, as well as applications in real world biomedical datasets in \cref{sec:application}. Lastly, we conclude with discussions and potential generalizations in \cref{sec:discussion}.

\subsection{Related Work}\label{subsec:related}
Random forests were first proposed by \citeasnoun{ho1995random}, followed by \citeasnoun{breiman2001random} providing theoretical guarantees for generalization error. \citeasnoun{friedman2008predictive} generalized the idea of rule ensembles in random forest, discussed variable importance assessment, and presented techniques to identify interactions among variables.

Numerous improvements have been made since the original debut of random forests. \citeasnoun**{robnik2004improving} proposed using attribute evaluation measures as split selectors to decrease correlation of trees in the forest. \citeasnoun**{bernard2009selection} showed that the tree ensembles can have better performance with subsets of decision trees. Since then, a lot of effort has been aimed at limiting the number of trees and finding the optimal subset of trees, also referred to as \textit{pruning} of a random forest. \citeasnoun**{paul2018improved} proposed a method that simultaneously performs feature selection based on importance ranking, and finds the optimal number of trees based on classification accuracy. The most relevant work to our paper includes the hedged random forest proposed by \citeasnoun**{beck2024hedged} that re-weights trees by minimizing the estimated mean squared error, the weighted random forest by \citeasnoun**{winham2013weighted} that utilizes the prediction error reduction as weights of trees, and the post-selection boosting random forest by \citeasnoun**{wang2021improving} that adds a Lasso regression step after obtaining individual tree fits.

Another related direction is using nonparametric models, such as random forest, to perform variable selection. \citeasnoun{huang2010variable} discussed using adaptive B-spline basis expansion and group
Lasso to perform variable selection in nonparametric additive models. \citeasnoun{chouldechova2015generalized} used blockwise coordinate descent to fit sparse generalized additive
models. Other major generalizations include the generalized random forest proposed by \citeasnoun{athey2019generalized}, which extends the application of random forest beyond standard conditional mean estimation.

\section{Methodology} \label{sec:methodology}
\subsection{Review of Post-selection Forests}
Standard random forests aggregate the tree predictions fit to bootstrapped data. Each tree draws a bootstrap sample of size $N$ from the training data and treats them as in-bag samples. Then, it recursively selects a fixed number of variables at random, pick the best split variable and splits the nodes until the minimum node size is reached. The subset of samples not used in training are out-of-bag samples, which are often used to assess the model.

The idea of a post-selection forest, or post-selection boosting random forest as in \citeasnoun{wang2021improving}, is to perform Lasso regression on the fitted tree predictions, and hence select a subset of trees to be included in the final prediction aggregation. It is designed to reduce the bias of each individual tree. 

Formally, suppose we have i.i.d. data $\mathcal{D} = \{(\mathbf{x}_1, y_1), (\mathbf{x}_2, y_2),...(\mathbf{x}_N, y_N)\}$, where $\mathbf{x}_i \in \mathbb{R}^p$ for $i=1,2,...N$. First, we use a vanilla forest to generate $J$ trees $\{T_j\}_{j=1}^J$ by bootstrap sampling. Next, instead of taking the simple average of all out-of-bag fitted tree values $\widehat{\mathbf{T}}_i^{oob}= \left\{\widehat{T}_{ij}|i \notin j^{th}\text{ bootstrap sample}, 1\leq j\leq J\right\}$, the post-selection forest treats all fitted predictions as features and applies an $l_1$-penalty on the weights for each tree. The weights solve the following optimization problem:
\begin{equation} \label{eq:LF}
    \mathbf{\gamma} = \argmin_\gamma L_1(\gamma)=\frac{1}{N}\sum_{i=1}^N\left(y_i-\gamma_0-\sum_{j=1}^J \widehat{T}_{ij}\gamma_j\right)^2+\lambda\sum_{j=0}^J|\gamma_j|
\end{equation}

In practice, $\lambda$ can be chosen by cross-validation. \autoref{fig:RFandLF} visualizes and compares the vanilla and post-selection forest. To avoid overfitting, we first divide the training set $\mathcal{D}$ into two halves, using the first half to train the random forest and the second half to train the Lasso selection. This procedure, which we call {\em cross-fitting} is detailed in  \cref{alg:LF}.

\begin{figure}[H]
     \centering
     \begin{subfigure}[b]{0.49\textwidth}
         \centering
         \includegraphics[width=\textwidth]{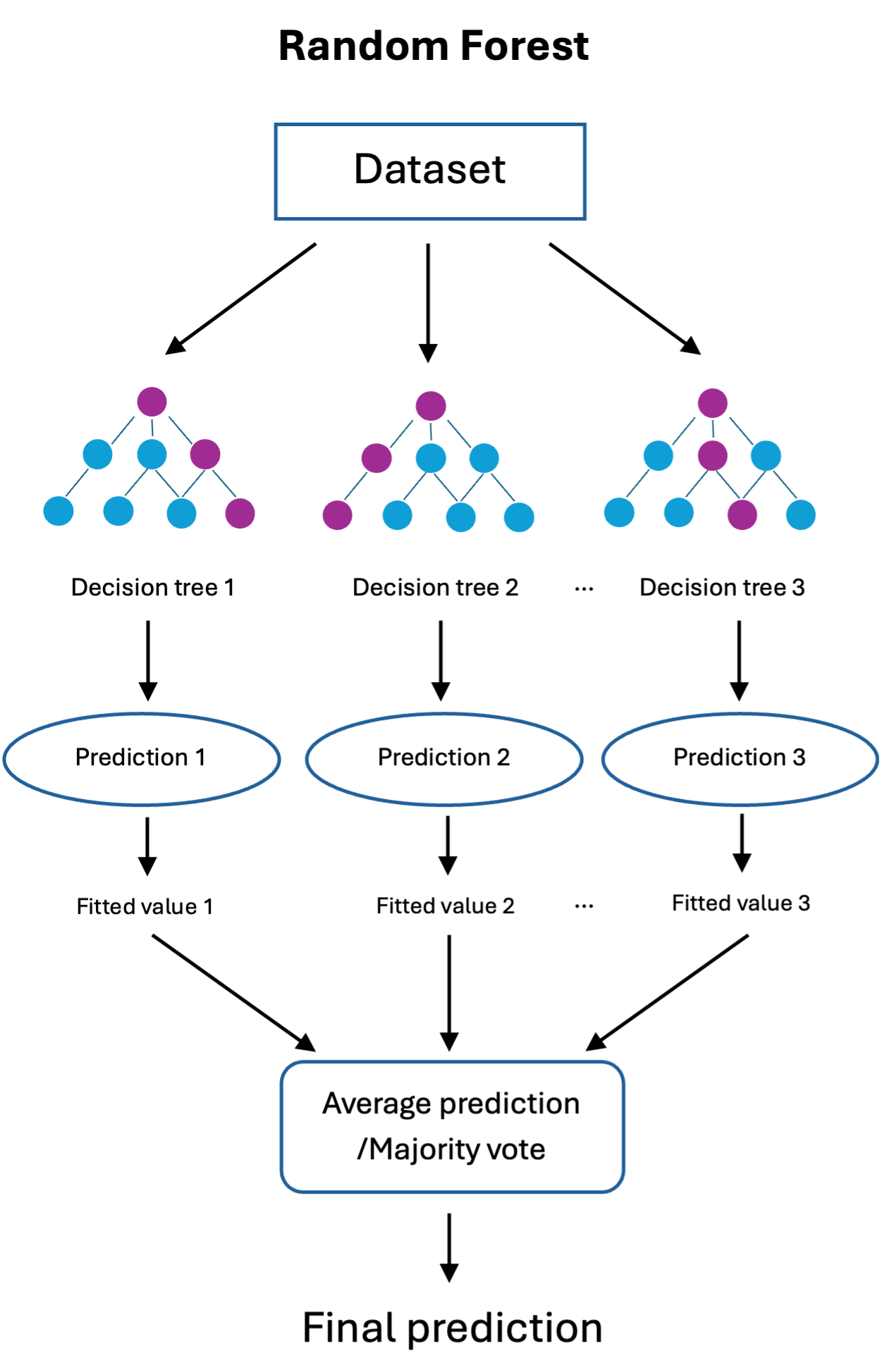}
         \caption{\em{Random Forest}}
         \label{fig:RF}
     \end{subfigure}
     \hfill
     \begin{subfigure}[b]{0.49\textwidth}
         \centering
         \includegraphics[width=\textwidth]{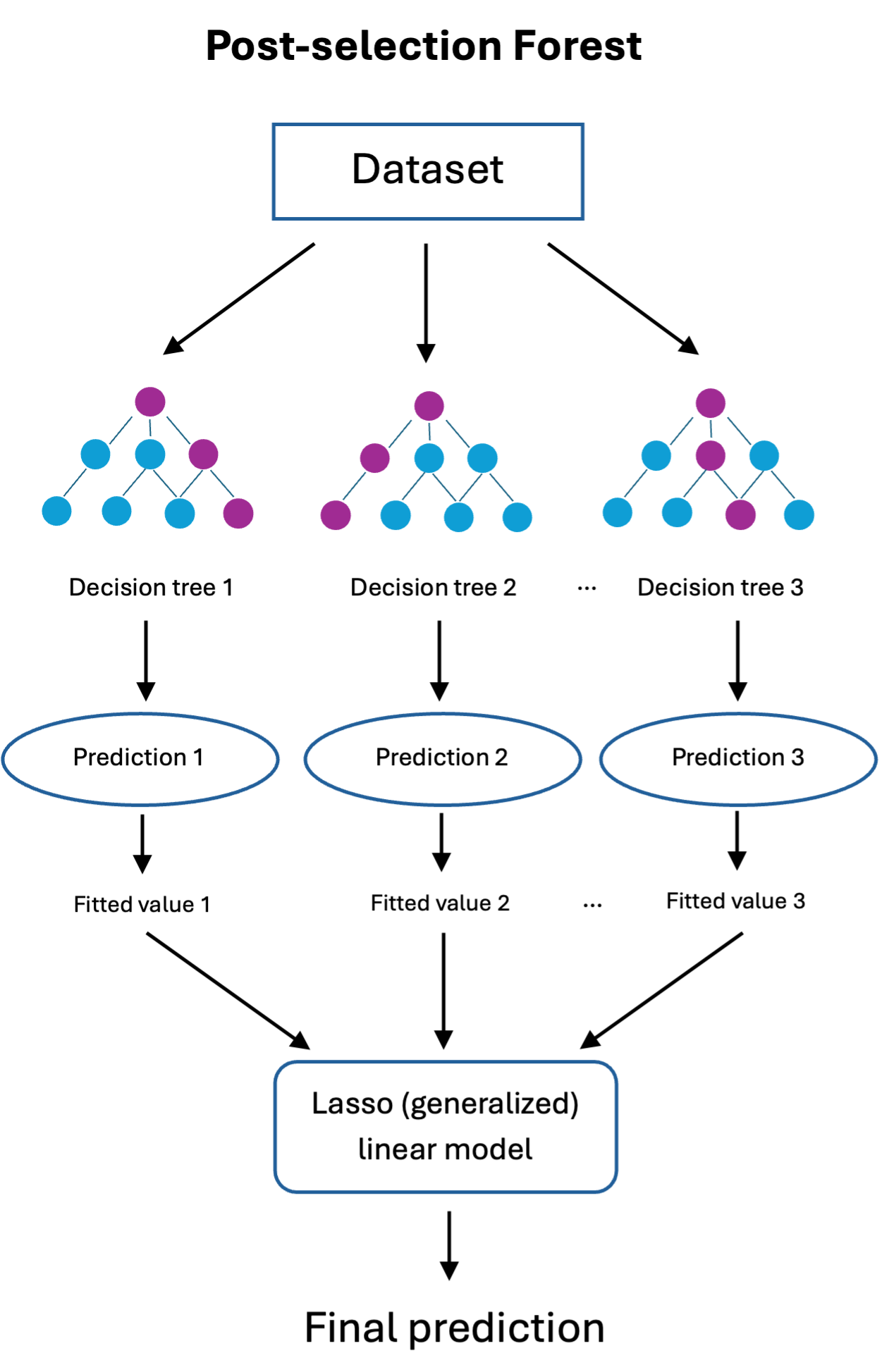}
         \caption{\em{Post-selection Forest}}
         \label{fig:PF}
     \end{subfigure}
        \caption{\em{Comparison of Random Forest and Post-Selection Forest}}
        \label{fig:RFandLF}
\end{figure}

\vspace{3mm}
\begin{algorithm}[H]
\setlength{\itemsep}{-1pt}
  \setlength{\parskip}{-1pt}
  \setlength{\parsep}{-1pt}
\caption{Post-selection Forest with Cross-fitting}\label{alg:LF}
\begin{algorithmic}
\State \textbf{Input:} training data $\mathcal{D}$ split into two halves $\mathcal{D}^1$ and $\mathcal{D}^2$, each of size N/2
\State
\For{$j$ in $1:J$}
\State \vspace{-5mm}

\begin{enumerate}
    \item Draw a bootstrap sample with replacement $Z^*$ of size $N/2$ from $\mathcal{D}^1$
    \item Grow a tree $T_j$ based on $Z^*$
    \item Obtain a vector of fitted value for each observation in $\mathcal{D}^2$: $\left(\widehat{T}_{1j}, \widehat{T}_{2j},...\widehat{T}_{(N/2)j}\right)$
\end{enumerate}
\EndFor
\State
\State $\hat{\gamma}\gets$ Lasso regression $y \sim \widehat{\mathbf{T}}$ by solving \eqref{eq:LF} \Comment{$y \in \mathcal{D}^2$, $\widehat{\mathbf{T}} \in \mathbb{R}^{(N/2)\times J}$ also for $\mathcal{D}^2$, }
\State \Comment{choose $\lambda$ by cross-validation}
\State
\State \textbf{Output:} given a new point $\mathbf{x}$, $\hat{y}\gets \left(\widehat{T}_1(\mathbf{x}),\widehat{T}_2(\mathbf{x}),...\widehat{T}_J(\mathbf{x})\right)\cdot \hat{\gamma}$ \Comment{evaluate trees at $\mathbf{x}$}
\end{algorithmic}
\end{algorithm}

\subsection{Our Proposal: Lassoed Forests}
As discussed in the introduction, we have observed that the performance of a post-selection forest can be worse than that of a vanilla forest under certain in some circumstances. Hence, we introduce the Lassoed forest that uses a weighted average of the vanilla forest and post-selection forest, and learns the optimal weights from data.

Formally, in addition to the set of weights $\{\gamma_j\}_{j=1}^J$ for each tree, we introduce another set of weights $\{\theta, 1-\theta\}$ to adaptively construct predictions from the vanilla and post-selection forests. For each fixed $\theta$, we solve:
\begin{equation} \label{eq:LARF}
    \mathbf{\gamma} = \argmin_\gamma L_2(\gamma)=\frac{1}{N}\sum_{i=1}^N\left(y_i-(1-\theta)\frac{1}{|O_i|}\sum_{j\in O_i} \widehat{T}_{ij}-\theta\gamma_0 -\theta \sum_{j=1}^J \widehat{T}_{ij}\gamma_j\right)^2+\lambda\sum_{j=0}^J|\gamma_j|
\end{equation}
where the aggregate out-of-bag prediction $\frac{1}{|O_i|}\sum_{j\in O_i} \widehat{T}_{ij}$ with weight $(1-\theta)$ serves as an offset term in Lasso regression, $O_i = \{j|i\notin Z^* \text{ used to train } T_j\}$. In other words, its coefficient is fixed at one. When $\theta = 0$ and $\theta = 1$, \eqref{eq:LARF} reduces to vanilla and post-selection forest respectively. 

We select the optimal $\hat{\theta}$ by comparing the cross-validation error across a grid of values for $\theta$. Note that when $\theta = 0$, \eqref{eq:LARF} only uses the vanilla forest predictions, where the prediction error could be otherwise estimated by out-of-bag error. \autoref{fig:LARF} illustrates the Lassoed forest procedure. As before, we use cross-fitting to train the vanilla forest and Lasso regression, in order to avoid overfitting at the second stage. \cref{alg:LARF} presents the method with cross-fitting in detail.

\begin{figure}[H]
    \centering
    \includegraphics[width=0.43\linewidth]{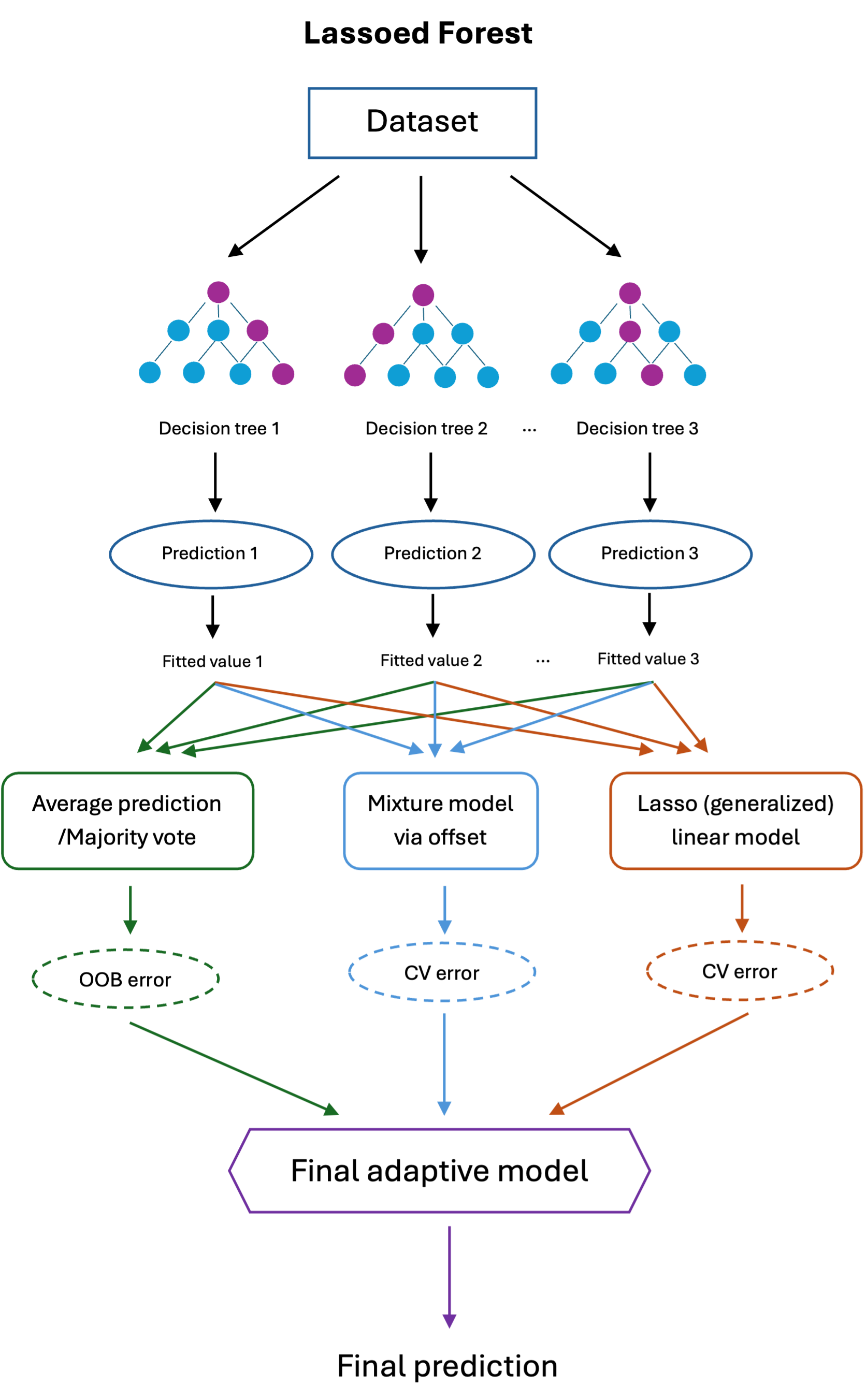}
    \caption{\em{Lassoed Forest}}
    \label{fig:LARF}
\end{figure}

\begin{algorithm}[H]
\caption{Lassoed Forest with Cross-fitting}\label{alg:LARF}
\begin{algorithmic}  \setlength{\itemsep}{-1pt}
  \setlength{\parskip}{-1pt}
  \setlength{\parsep}{-1pt}
\State \textbf{Input:} training data $\mathcal{D}$ split into two halves $\mathcal{D}^1$ and $\mathcal{D}^2$, each of size N/2 \footnotemark

\State
\For{$j$ in $1:J$}
\State \vspace{-5mm}
\begin{enumerate}
    \item Draw a bootstrap sample $Z^*$ of size $N/2$ from $\mathcal{D}^1$
    \item Grow a tree $T_j$ based on $Z^*$
    \item Obtain a vector of fitted value for each observation in $\mathcal{D}^2$: $\left(\widehat{T}_{1j}, \widehat{T}_{2j},...\widehat{T}_{(N/2)j}\right)$
\end{enumerate}
\EndFor
\State
\For{$\theta$ in $[0,1]$} 
\State \vspace{-5mm}
\begin{enumerate}
    \item $\hat{\gamma}\gets$ Lasso regression $y-(1-\theta)\overline{\mathbf{T}}\sim \theta\widehat{\mathbf{T}}$ by solving \eqref{eq:LARF} \Comment{using data in $\mathcal{D}^2$}
    \item Calculate cross-validation error $\hat{e}_{cv}$ for Lasso fit \Comment{out-of-bag error if $\theta = 0$}
\end{enumerate}
\EndFor
\State
\State $\hat{\theta}\gets \arg\min_\theta\hat{e}_{cv}(\theta)$
\State
\State \textbf{Output:} given a new point $\mathbf{x}$, $\hat{y}\gets \hat{\theta} \cdot \left(\widehat{T}_1(\mathbf{x}),\widehat{T}_2(\mathbf{x}),...\widehat{T}_J(\mathbf{x})\right)\cdot \hat{\gamma} + (1-\hat{\theta})\cdot\overline{\mathbf{T}}$ 
\end{algorithmic}
\end{algorithm}
\footnotetext{The half-splitting is appropriate for a moderate-sized dataset, e.g. $n=144$ in our immune checkpoint case study in \autoref{sec:application}.}

\section{Theory} \label{sec:theory}
\subsection{Signal-to-noise Ratio Dependency}
The motivating examples in \cref{subsec:examples} suggest that there could be underlying factors that dominate the relative performance of vanilla forest and post-selection forest. In this section, we show in theory that the signal-to-noise ratio determines their relative performance. As mentioned in \cref{subsec:related}, if we treat each tree as a base learner, then the post-selection forest given by \eqref{eq:LF} is just a special case of rule ensembles \cite{friedman2008predictive}. In this section, we show results that are model-free and can be generalized to post-selection of any base learners satisfying minimum requirements.

\vspace{2mm}
We start with the following model-agnostic assumption on the data:
\vspace{-1mm}
\begin{assumption}
    Assume $\mathcal{D} = \{(\mathbf{x}_i, y_i)\}_{i=1}^N$ are i.i.d. data, with $y_i = g(\mathbf{x}_i)+ \epsilon_i$, $\epsilon_i \perp\!\!\!\perp \mathbf{x}_i$. We assume $\epsilon_i$ follows some distribution with mean $0$ and variance $\sigma^2$. We also assume $\mathbf{x}_i$ has covariance matrix $\Sigma$, its function $g(\mathbf{x}_i)$ has mean $\mu$ and variance $\phi(\Sigma)$. We define the signal-to-noise ratio (SNR) as:
    \begin{equation*}
        s:=\frac{\text{Var}(g(\mathbf{x}_i))}{\sigma^2}=\frac{\phi(\Sigma)}{\sigma^2}.
    \end{equation*}We keep the signal level $\phi(\Sigma)$ fixed, and increase $s$ by decreasing $\sigma^2$.
\end{assumption}

\begin{remark} \label{rm:SNR}
For the signal-to-noise ratio $s$,
\begin{enumerate}
    \item It is defined more generally without modeling assumptions. When the underlying model is linear $y_i = \mathbf{x}_i^T\beta+\epsilon_i$, it agrees with the standard form $s = \frac{\text{Var}(\mathbf{x}_i^T\beta)}{\sigma^2} = \frac{\beta^T\Sigma\beta}{\sigma^2}$; 
    \item It is invariant to constant location shifts and scaling in response $y$. In practice, while neither location shifts nor scaling changes the behavior of random forest, a non-parametric procedure, they indeed impact the quality of out-of-bag and cross-validation error estimates.
\end{enumerate}
\end{remark}

We define $\{f_j\}_{j=1}^J$ as exogenous base learners that are given or pre-trained, i.e. do not depend on $\mathcal{D}$. We do not make any assumptions on independency, as is the case in training on bootstrap samples. We are interested in comparing the two predictions for a given new independent point $\mathbf{x}^*$ from the same joint distribution, with $y^* = g(\mathbf{x}^*)+ \epsilon$:
\begin{align}
    \hat{y}^{mean} &=\frac{1}{J}\sum_{j=1}^J f_j(\mathbf{x}^*)\label{eq:est_mean}\\
     \hat{y}^{reg} &= \hat{\gamma}_0 +\sum_{j=1}^J \hat{\gamma}_j f_j(\mathbf{x}^*) \label{eq:est_reg}
\end{align}
where $\hat{\gamma}_j = \argmin_{\gamma_j} \frac{1}{N}\sum_{i=1}^N\left(y_i - \gamma_0-\sum_{j=1}^J \gamma_j f_j(\mathbf{x}_i)\right)^2 + \lambda ||\gamma ||_1$. We assume that all base learners are from the function class. This is a common case in model aggregating and boosting, such as random forest, Adaboost \cite{freund2001adaptive}, etc. We make the following further assumptions about the base learners:

\begin{assumption} \label{assump:equibias}\textup{\textbf{(Equibias)}} We assume the base learners are conditionally equibiased for a given finite sample size, and the bias term depends on the signal-to-noise ratio,
\begin{equation*}
\begin{aligned}
&\mathbb{E}[f_j(\mathbf{x}_i)|\mathbf{x}_i] = g(\mathbf{x}_i) + \eta(s),\quad \forall j \\
&\mathbb{E}[f_j] = \mu+ \eta(s),\quad \forall j
\end{aligned}
\end{equation*}
where $\eta(\cdot)$ is a monotonically non-increasing function of $s$ \cite{box1988signal}, without loss of generality assuming $\eta(\cdot)\geq0$.
\end{assumption}

\begin{assumption} \label{assump:homoscedasticity}\textup{\textbf{(Homoscedasticity)}} We assume the base learners share the same covariance, while still allow the conditional ones to depend on a given data point,
\begin{equation*}
\begin{aligned}
&\textup{Var}(f_j(\mathbf{x}_i)|\mathbf{x}_i) = \tau^2(\mathbf{x}_i, s), \qquad \mathbb{E}[\tau^2(\mathbf{x}_i, s)] = \psi(s),\quad \forall j\\
& \textup{Cov}(f_j(\mathbf{x}_i), f_k(\mathbf{x}_i)|\mathbf{x}_i) = \rho(\mathbf{x}_i, s), \qquad \mathbb{E}[\rho(\mathbf{x}_i, s)] = \omega(s),\quad \forall j<k
\end{aligned}
\end{equation*}
where $\tau^2(\cdot)$, $\psi(\cdot)$, $\rho(\cdot)$ and $\omega(\cdot)$ are all monotonically non-increasing functions of $s$ \cite{box1988signal}. We also assume $\rho$ and $\omega$ are lower bounded by 0.
\end{assumption}

\begin{assumption} \label{assump:expressivity}\textup{\textbf{(Expressivity)}} We assume the base learners are flexible in the following way, 
\begin{equation*}
\begin{aligned}
&\mathbb{E}\left[y_i|\{f_j(\mathbf{x}_i)\}_{j=1}^J\right] = \gamma_0+\sum_{j=1}^J \gamma_j f_j(\mathbf{x}_i) \\
&\textup{Var}\left(y_i|\{f_j(\mathbf{x}_i)\}_{j=1}^J\right) = \sigma^2
\end{aligned}
\end{equation*}
which implies $g(\mathbf{x}_i) = \gamma_0+\sum_{j=1}^J \gamma_j f_j(\mathbf{x}_i)$, $\forall i$, and $\gamma_0 = -\eta(s)$, $\sum_{j=1}^J\gamma_j = 1$. It is a reasonable assumption for learners that are well known to have low bias but high variance, such as random forests \cite**{hastie2009random} or neural networks.
\end{assumption}

These assumption lead to the following fact about the base learners:
\begin{proposition} \label{prop:SNR}
    Under \cref{assump:equibias,assump:homoscedasticity,assump:expressivity}, we additionally assume both the expectation and covariance of Lasso coefficients conditioning on pre-trained base learners are monotonically non-increasing in SNR. Then the mean squared error for both model predictions \eqref{eq:est_mean} and \eqref{eq:est_reg} are also monotonically non-increasing functions of SNR, but have different dependency on the SNR.
\end{proposition}
\begin{proof}
A sketch of the proof is shown here. The full proof can be found in \cref{appendix:SNR}. 

\vspace{1mm}
For $\hat{y}_i^{mean}$, it is biased by $\eta(s)$ under \cref{assump:equibias}. Its variance can be decomposed by law of conditional variance as:
\begin{equation*}
\begin{aligned}
    \text{Var}(\hat{y}^{mean}) &= \text{Var}\left(\frac{1}{J}\sum_{j=1}^Jf_j(\mathbf{x}^*)\right)\\
    &=\frac{1}{J}\mathbb{E}\left[\text{Var}(f_j(\mathbf{x}^*)|\mathbf{x}^*)\right] +\frac{J(J-1)}{J^2}\mathbb{E}\left[\text{Cov}(f_j(\mathbf{x}^*),f_k(\mathbf{x}^*)|\mathbf{x}^*)\right] +\text{Var}\left(\mathbb{E}[f_j(\mathbf{x}^*)|\mathbf{x}^*]\right) 
\end{aligned}
\end{equation*}
where the first two terms are monotonically non-increasing functions of the SNR by \cref{assump:homoscedasticity}, and the third term is a function of $\Sigma$. Intuitively, the first two terms measures the uncertainty of fitting base learners, while the third term shows the variance from data itself.

\vspace{1mm}
For $\hat{y}^{reg}$, if we assume the oracle $\gamma_j$ as in \cref{assump:expressivity} are known, then $\hat{y}^{reg}$ is indeed unbiased. Variance can be computed as follows:
\begin{equation*}
\begin{aligned}
\text{Var}(\hat{y}^{reg}) =\text{Var}(g(\mathbf{x}^*))
\end{aligned}
\end{equation*}
which is a function of $\Sigma$. Extra uncertainty would be introduced if $\gamma_j$ are estimated by $\hat{\gamma}_j$, which we show in the full proof that it is still a monotonically non-increasing function of the SNR.
\end{proof}

\begin{remark}
    Under the oracle case, given we know that $\sum_{j=1}^J\gamma_j = 1$ from \cref{assump:expressivity}, we show:
\begin{equation*}
\sum_{j=1}^J\gamma_j^2 \geq \frac{1}{J}\;.
\end{equation*}
Note that 
\begin{equation*}
\begin{aligned}
    \text{Var}(\hat{y}^{reg}) =& \sum_{j=1}^J\gamma_j^2\mathbb{E}\left[\text{Var}(f_j(\mathbf{x}^*)|\mathbf{x}^*)\right] \\&+ 2\sum_{j<k}^J\gamma_j\gamma_k\mathbb{E}\left[\text{Cov}(f_j(\mathbf{x}^*),f_k(\mathbf{x}^*)|\mathbf{x}^*)\right]+\text{Var}\left(\mathbb{E}[f_j(\mathbf{x}^*)|\mathbf{x}^*]\right)\;,
\end{aligned}
\end{equation*}
where we generally expect $\text{Cov}(f_j(\mathbf{x}^*),f_k(\mathbf{x}^*)|\mathbf{x}^*) > 0$, $\forall j <k$. Hence, comparing with the formula of $\text{Var}(\hat{y}^{mean})$, it shows that the necessary conditions for $\hat{y}^{reg}$ to achieve variance reduction are (a) allowing negative coefficients $\gamma_j$; and (b) assigning $\gamma_j$, $\gamma_k$ in a way such that $2\sum_{j<k}^J\gamma_j\gamma_k\ll\frac{J(J-1)}{J^2}$. Meanwhile, bias reduction can be achieved by introducing a debiasing intercept term $\gamma_0$ to the regression.
\end{remark}

In summary, we have shown in this section that the performance of vanilla forest and post-selection forest are directly related to the signal-to-noise raio, but in different forms. This corresponds to the change of relative performance as SNR grows. Our result not only applies to tree models, but also to a general class of base learners.

\subsection{Bias-variance Tradeoff} \label{subsec:tradeoff}
In this section, we further investigate the properties of prediction mean squared error for both vanilla forest and post-selection forest via the bias-variance decomposition. We first look into a special case where the selection procedure is a simple linear regression without regularization. While this corresponds to the procedure of post-weighting, rather than post-selection,  the closed-form solution to \eqref{eq:est_reg} when $\lambda = 0$ makes the analysis more tractable.

We plug in the OLS estimator $\hat{\Gamma} = (F^TF)^{-1}F^Ty$ to $\Gamma = (\gamma_0, \gamma_1,...\gamma_J)$, where $F = (F_1,...F_N)^T=(1,f_1,...f_J) \in \mathbb{R}^{N \times (J+1)}$ contains all base learner predictions for $\mathcal{D}$. The following theorem brings more insights into the signal-to-noise ratio dependency by showing the exact form of mean squared error:

\begin{theorem} \label{thm:tradeoff}
    Under \cref{assump:equibias,assump:homoscedasticity,assump:expressivity}, let $\mathcal{F}= (f_1,...f_J)$ be the pre-trained base learners, $J+1<N$, assuming in addition that the outputs conditioning on base learners $(f_1(\mathbf{x}),...f_J(\mathbf{x}))|\mathcal{F} \sim \mathcal{N}(\mathbf{0}, \mathbf{W})$. Then the prediction mean squared error $\mathbb{E}\left[(y^*-\hat{y}^{mean})^2\right] = \eta(s)^2 + \frac{1}{J} \psi(s) +\frac{J(J-1)}{J^2}\omega(s) +\phi(\Sigma)$, while $\mathbb{E}\left[(y^*-\hat{y}^{reg})^2\right]=\frac{\sigma^2}{N-J-1}+ \phi(\Sigma)$.
\end{theorem}

\begin{proof}
As before, a sketch of proof is shown here. We show the case where $J+1<N$, while the full proof in \cref{appendix:tradeoff} also discusses the case where $J>N$ in detail.

\vspace{1mm}
For $\hat{y}^{mean}$, it has the exact same properties as in \cref{prop:SNR}: 
\begin{equation*}
    \begin{aligned}
    \mathbb{E}[\hat{y}^{mean}] &= \mu + \eta(s)\\
    \text{Var}(\hat{y}^{mean}) &=\frac{1}{J}\psi(s) +\frac{J(J-1)}{J^2}\omega(s)+ \phi(\Sigma)
    \end{aligned}
\end{equation*}

For $\hat{y}^{reg}$, we analyze its bias and variance by first conditioning on all the base learner predictions $\mathcal{P} = \left\{\left(f_1(\mathbf{x}_i), ...f_J(\mathbf{x}_i)\right)\right\}_{i=1}^N$ on $\mathcal{D}$. Denote the base learner prediction matrix for $\{\mathbf{x}_i\}_{i=1}^N$ by $F = \left(1, f_1(\mathbf{x}), ...f_J(\mathbf{x})\right)\in \mathbb{R}^{N\times(J+1)}$, and prediction vector for $\mathbf{x}^*$  by $F^* = \left(1, f_1(\mathbf{x}^*), ...f_J(\mathbf{x}^*)\right)\in \mathbb{R}^{J+1}$:
\begin{equation*}
    \begin{aligned}
    \mathbb{E}[\hat{y}^{reg}] &= \mathbb{E}\left[\mathbb{E}[F^*(F^TF)^{-1}F^Ty|\mathcal{P}, \mathbf{x}^*]\right]
    =\mathbb{E}[F^*\Gamma] = \mu\\
    \text{Var}(\hat{y}^{reg}) &=\mathbb{E}\left[\text{Var}(F^*(F^TF)^{-1}F^Ty|\mathcal{P},\mathbf{x}^*)\right] + \text{Var}\left(\mathbb{E}[F^*(F^TF)^{-1}F^Ty|\mathcal{P},\mathbf{x}^*]\right) \\
    &=\mathbb{E}\left[\text{Var}(F^*(F^TF)^{-1}F^T\epsilon|\mathcal{P},\mathbf{x}^*)\right] + \text{Var}\left(\gamma_0 + \sum_{j=1}^J \gamma_j f_j(\mathbf{x}^*)\right)\\
    &=\sigma^2 \mathbb{E}\left[F^*(F^TF)^{-1}F^{*T}\right]+ \text{Var}(g(\mathbf{x}^*))\\
    &=\frac{\sigma^2}{N-J-1}+ \phi(\Sigma)
    \end{aligned}
\end{equation*}
where for bias calculation, the second equality holds by independence of $\mathbf{x}^*$ and \cref{assump:expressivity}, the third equality holds by \cref{assump:equibias} and \cref{assump:expressivity}; for variance calculation, the second and third equality hold by \cref{assump:expressivity}, the last equality holds by applying the trace trick and conditioning on the pre-trained base learners $\mathcal{F}$.
\end{proof}

\begin{corollary}
    If we further assume that $\eta(s)=0$ in \cref{assump:equibias}, i.e. all base learners are conditionally unbiased, and keep the signal level $\phi(\Sigma)$ fixed, then the comparison of mean squared error depends on the number of base learners $J$, sample size $N$, signal-to-noise ratio $s$, variance and covariance of base learners $\psi(s)$, $\omega(s)$.
\end{corollary}

\begin{remark} \label{rm:tradeoff} Comparing the variance and bias terms for the case $\lambda = 0$, i.e. post-weighting by linear regression, we can draw the following conclusions:
\begin{enumerate}
    \item $\hat{y}^{mean}$ does not reduce bias from the base learners, while $\hat{y}^{reg}$ can remove any fixed biases that are allowed to be dependent on the sample size and SNR.
    \item While both $\hat{y}^{mean}$ and $\hat{y}^{reg}$ share the same variance $\phi(\Sigma)$ from data, variance from fitting base learners scales differently with SNR. For $\hat{y}^{mean}$, $\psi(s)$ and  $\omega(s)$ both depend on the particular choice of base learners. For $\hat{y}^{reg}$, $\frac{\phi(\Sigma)}{s(N-J-1)}$ has fixed scale $O\left(\frac{1}{s}\right)$. This suggests that post-weighting benefits from utilizing the correlation between individual trees to stabilize model performance. We give some simple examples of $\psi(s),\omega(s)$ under specific modeling assumptions in \cref{appendix:variance}. For instance, if the base learners are linear models and correctly specified, then $\psi(s),\omega(s) = O(\frac{1}{s})$, in which case the variance comparison only depends on $J$ and $N$, i.e. number of base learners versus sample size.
\end{enumerate}
\end{remark}

Now we consider the case when $\lambda > 0$ and provide the intuition of MSE reduction only. Compared with linear regression, especially in the case where $J\geq N$ as discussed in \cref{appendix:tradeoff}, applying Lasso re-introduces bias, but reduces variance. On one hand, variance reduction is still a non-increasing function of SNR with scale $O(\frac{1}{s^{(q+p)}})$, assuming $\psi(s),\omega(s) = O(\frac{1}{s^p})$, $p$, $q\geq0$. The intuition behind the double reduction is that if the variance from training base learners scale as $O(\frac{1}{s^p})$, then the $l_1$-regularization could further reduce variance by scale $O(\frac{1}{s^q})$. On the other hand, previous work by \citeasnoun{4839045}, \citeasnoun{10.1214/07-AOS582}, etc. has proved the support recovery for Lasso under high-dimensional settings. Hence, we conjecture that compared with the vanilla forest, performing post-selection with Lasso could suffer from variance increase, depending on the relative magnitude of uncertainties from fitting base learners, but bring benefits of bias reduction.

\subsection{Improvement via Adaptive Objective with Offset} \label{subsec:offset}
In this section, we show that our proposal of Lassoed forest is a strict improvement to both vanilla forest and post-selection forest. In other words, despite the obvious fact that vanilla forest and post-selection forest are two special cases of the adaptive objective \eqref{alg:LARF} with $\theta = 0$ and $\theta = 1$, we show that a choice of $\theta \in (0,1)$ could give strictly better prediction mean squared error under certain circumstances. The results are readily generalizable to any base learners satisfying \cref{assump:equibias,assump:homoscedasticity,assump:expressivity}. Again, we analyze the case when $\lambda = 0$ and provide only intuition for the case when $\lambda > 0$.

\vspace{3mm}
First, when $\lambda = 0$, the solution to \eqref{eq:LARF} is:

\begin{equation*}
    \hat{y}^{ada}=(1-\theta)\frac{1}{J}\sum_{j=1}^J \widehat{T}^*_j+\theta\sum_{j=0}^J  \widehat{T}^*_j\left((\widehat{\mathbf{T}}^T\widehat{\mathbf{T}})^{-1}\widehat{\mathbf{T}}^Ty\right)_j
\end{equation*}
where $\widehat{\mathbf{T}}= (\widehat{\mathbf{T}}_1,...\widehat{\mathbf{T}}_N)^T=(1, \widehat{T}_{\cdot 1},...\widehat{T}_{\cdot J}) \in \mathbb{R}^{N/2 \times (J+1)}$ are tree predictions by cross-fitting, $\widehat{T}^*= (1,\widehat{T}^*_1,...\widehat{T}^*_J)$ are tree predictions for a new given point. The theorem below provides a direct comparison between all three methods:

\begin{theorem} \label{thm:adaptive}
    Under \cref{assump:equibias,assump:homoscedasticity,assump:expressivity}, assuming in addition that $(f_1(\mathbf{x}),...f_J(\mathbf{x}))|\mathcal{F} \sim \mathcal{N}(\mathbf{0}, \mathbf{W})$, where $\mathcal{F}= (f_1,...f_J)$ are pre-trained base learners, $J+1<N$. Then $\mathbb{E}\left[(y^*-\hat{y}^{ada})^2\right] = (1-\theta)^2\eta(s)^2+(1-\theta)^2\left(\frac{1}{J}\psi(s) +\frac{J(J-1)}{J^2}\omega(s)\right)+\theta^2\frac{\sigma^2}{N-J-1}+ \phi(\Sigma)$.
\end{theorem}

\begin{proof}
The full derivation can be found in \cref{appendix:adaptive}.

\vspace{1mm}
By linearity of expectation, 
    \begin{equation*}
    \mathbb{E}[\hat{y}^{ada}] = (1-\theta) \mathbb{E}[\hat{y}^{mean}] + \theta\mathbb{E}[\hat{y}^{reg}]=\mu+(1-\theta)\eta(s)
    \end{equation*}

By law of conditional variance,

    \begin{equation*}
    \begin{aligned}
    \text{Var}(\hat{y}^{ada}) =&\mathbb{E}\left[\text{Var}\left((1-\theta) \overline{T^*} +\theta\widehat{T}^*(\widehat{\mathbf{T}}^T\widehat{\mathbf{T}})^{-1}\widehat{\mathbf{T}}^Ty\Big|\mathbf{x}^*\right)\right]\\ &+ \text{Var}\left(\mathbb{E}\left[(1-\theta) \overline{T^*} +\theta\widehat{T}^*(\widehat{\mathbf{T}}^T\widehat{\mathbf{T}})^{-1}\widehat{\mathbf{T}}^Ty\Big|\mathbf{x}^*\right]\right) \\
    =&\mathbb{E}\left[\text{Var}\left((1-\theta) \overline{T^*} + \theta\widehat{T}^*(\widehat{\mathbf{T}}^T\widehat{\mathbf{T}})^{-1}\widehat{\mathbf{T}}^T\epsilon\Big|\mathbf{x}^*\right)\right]+ \text{Var}\left(g(\mathbf{x}^*)\right)\\
    =&(1-\theta)^2\left(\frac{1}{J}\psi(s) +\frac{J(J-1)}{J^2}\omega(s)\right)+\theta^2\frac{\sigma^2}{N-J-1}+ \phi(\Sigma)
    \end{aligned}
\end{equation*}
\end{proof}

\begin{remark} \label{rm:adaptive}
    Comparing with results from \cref{thm:tradeoff}, we see that:
    \begin{enumerate}
        \item The adaptive random forest prediction $\hat{y}^{ada}$ without regularization still carries the same uncertainty $\phi(\Sigma)$ from the data.
        \item However, $\hat{y}^{ada}$ imposes a shrinkage on the bias term $\eta(s)$ compared with the vanilla forest. Meanwhile, it also balances the variance from fitting base learners between vanilla forest and post-selection forest. Most importantly, the balancing weights lie in the interior of the convex hull of $\left\{\eta(s)^2 +\frac{1}{J}\psi(s) +\frac{J(J-1)}{J^2}\omega(s), \frac{\sigma^2}{N-J-1}\right\}$, as 
        \begin{equation*}
            (1-\theta)^2 +\theta^2 <1, \quad\text{when }\theta \in (0,1).
        \end{equation*}
        This implies that there always exists $\theta \in (0,1)$, such that
        \begin{equation*}
            \mathbb{E}\left[(y^*-\hat{y}^{ada})^2\right]<\min\left\{\eta(s)^2 + \frac{1}{J} \psi(s) +\frac{J(J-1)}{J^2}\omega(s), \frac{\sigma^2}{N-J-1}\right\}+\phi(\Sigma).
        \end{equation*}
        For example, suppose $\eta(s)^2 + \frac{1}{J} \psi(s) +\frac{J(J-1)}{J^2}\omega(s)= \frac{\sigma^2}{N-J-1}$, then any $\theta \in (0,1)$ satisfies the strict inequality. More generally, suppose $\frac{\eta(s)^2 + \psi(s)/J +J(J-1)\omega(s)/J^2}{\sigma^2/(N-J-1)}= c\leq 1$, then one can always pick $\theta < \min\left\{\frac{2c}{c+1} ,1\right\}$ so that the inequality holds.
    \end{enumerate}
\end{remark}

\vspace{3mm}
Now we discuss the case when $\lambda >0$. We have previously conjectured that applying Lasso increases bias, and the variance comparison with vanilla forest depends on the relative magnitude of uncertainties from fitting base learners. However, we argue that the balancing effect from applying adaptive weights always guarantees an admissible mean squared error, i.e. at least as good as both vanilla and post-selection forest. This is due to the quadratic form of weights on a 1-dimensional simplex $\{\theta^2, (1-\theta)^2\}$, and is not impacted by the absolute magnitude of mass located at each vertex.

In practice, the Lassoed forest almost certainly have the global optimal performance among all three methods, though tuning the weights $\{\theta, 1-\theta\}$ based on out-of-bag and cross-validation error estimates could bring some costs.

\section{Simulation studies} \label{sec:simulation}
In this section, we provide multiple simulation experiments with different data generating processes. We investigate whether the signal-to-noise ratio dependency for random forests is universal, and whether applying the Lassoed forests is beneficial. We use the R package \texttt{ranger} \cite{wright2017ranger} to train random forests and run analyses.

\vspace{1mm}
For each experiment, we do the following:
\vspace{-1mm}
\begin{enumerate}
    \item compare the model performance under various SNRs;
    \item test the accuracy of out-of-bag and cross-validation error estimates that are used to pick adaptive weights $\{\theta, 1-\theta\}$;
    \item decompose the mean squared error to verify theorems and conjectures numerically;
    \item analyze the variable importance by measures based on split counts.
\end{enumerate}

\subsection{Polynomial Functions as Input}
First, we consider the case where the data generating function is polynomial:
\begin{equation*}
    y = \sum_{j=1}^p \alpha_jx_j + \sum_{1\leq j<k\leq p} \beta_{j,k} x_jx_k +\epsilon
\end{equation*}\vspace{-1mm}where $x_j\overset{i.i.d.}{\sim} N (0, 1)$, $\epsilon \overset{i.i.d.}{\sim} N (0, \sigma^2)$. Both $\boldsymbol{\alpha}$ and $\boldsymbol{\beta}$ are constructed to be sparse vectors as follows, with $c=0.1$ and $\pi=0.5$: 
\begin{equation*}
    \begin{aligned}
    \alpha_j &= a \cdot \text{Unif}(0,c) + (1-a) \cdot 0, \quad a\overset{i.i.d.}{\sim} \text{Bern}(\pi)\\
    \beta_{j,k} &=b \cdot \text{Unif}(0,c)+ (1-b) \cdot 0, \quad b\overset{i.i.d.}{\sim} \text{Bern}(\pi)
    \end{aligned}
\end{equation*}

We set the sample size $n=400$, number of trees $J=200$, number of features $p=50$. \autoref{fig:polynomial_MSE} shows the comparison of prediction mean squared error among vanilla forest, post-selection forest, and Lassoed forest under different signal-to-noise ratios. We scale and center the response $y$ before fitting the model, $\theta$ is chosen from $\{0,0.25,0.5,0.75,1\}$.
\begin{figure}[H]
    \centering
    \includegraphics[width=1\linewidth]{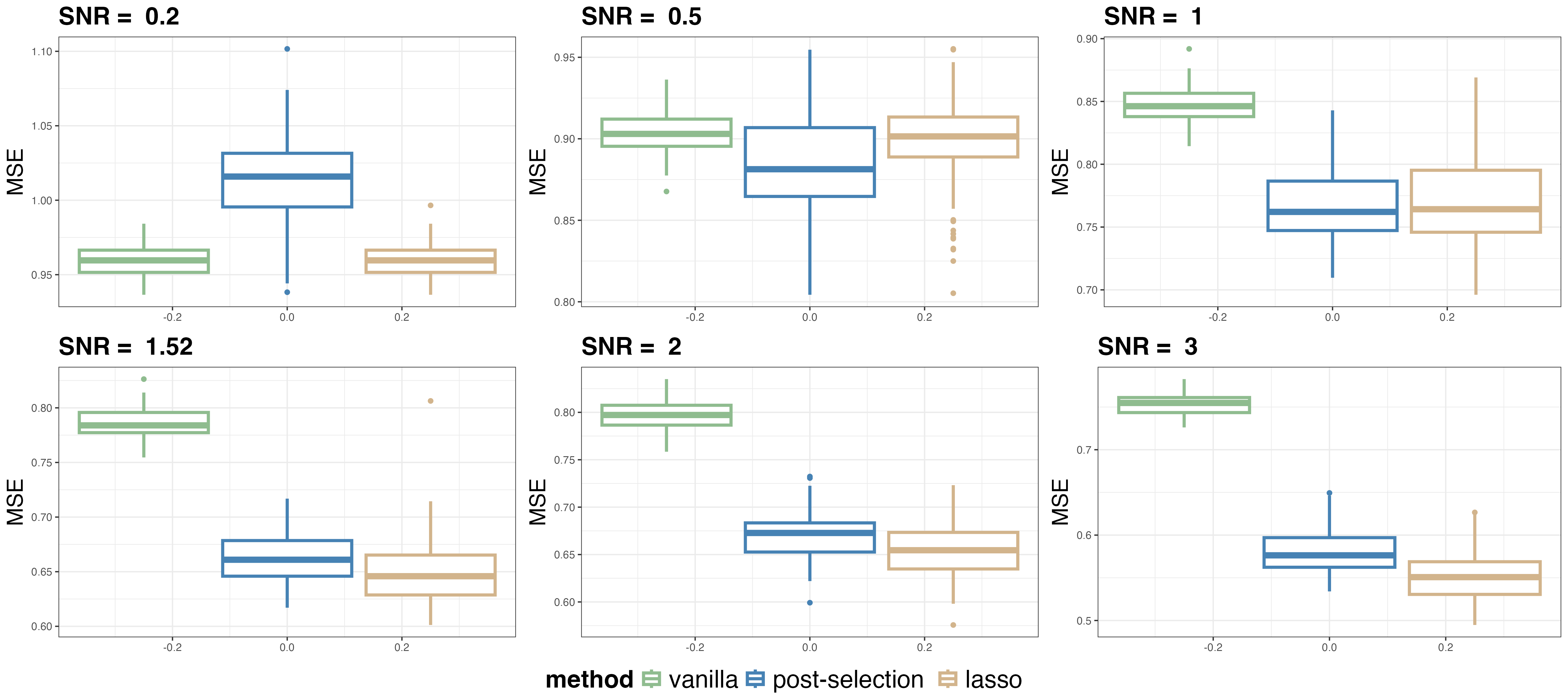}
    \caption{\em{Mean Squared Error for Polynomial Generating Functions}}
    \label{fig:polynomial_MSE}
\end{figure}

We can observe the followings:
\begin{enumerate}
    \item As the signal-to-noise ratio increases, the post-selection forest starts to outperform the vanilla forest. The Lassoed forest can almost always correctly identify which method is preferred in a particular SNR setting by using error estimates, and picks the adaptive weights balancing the vanilla and post-selection forest.
    \item As we showed in \cref{thm:adaptive}, the Lassoed forest can theoretically achieve lower mean squared error than both the vanilla and post-selection forest. This is verified in the numerical experiments when SNR is in medium to high range.
    \item The Lassoed forest does not give the best performance in a few cases, especially when vanilla and post-selection forest have comparable performances. This is because the out-of-bag and cross-validation error estimates are roughly unbiased, but fairly noisy. In other words, the error estimates based on a held-out set can be seen as the true test error plus some random noise. Hence, the Lassoed forest is sometimes unable to pick the optimal weights.
\end{enumerate}

\autoref{fig:polynomial_error} further shows the accuracy of these error estimates. Here, we only display the out-of-bag estimates for vanilla forest, and the cross-validation estimates for post-selection forest, which corresponds to $\theta = 0$ and $\theta = 1$. We also plot the differences in true test errors against estimated errors. An ideal case would be most points falling in the $1^{st}$ and $3^{rd}$ quadrants, i.e. sharing the same sign, even if the estimates themselves are noisy. 

\begin{figure}[H]
    \centering
    \includegraphics[width=1\linewidth]{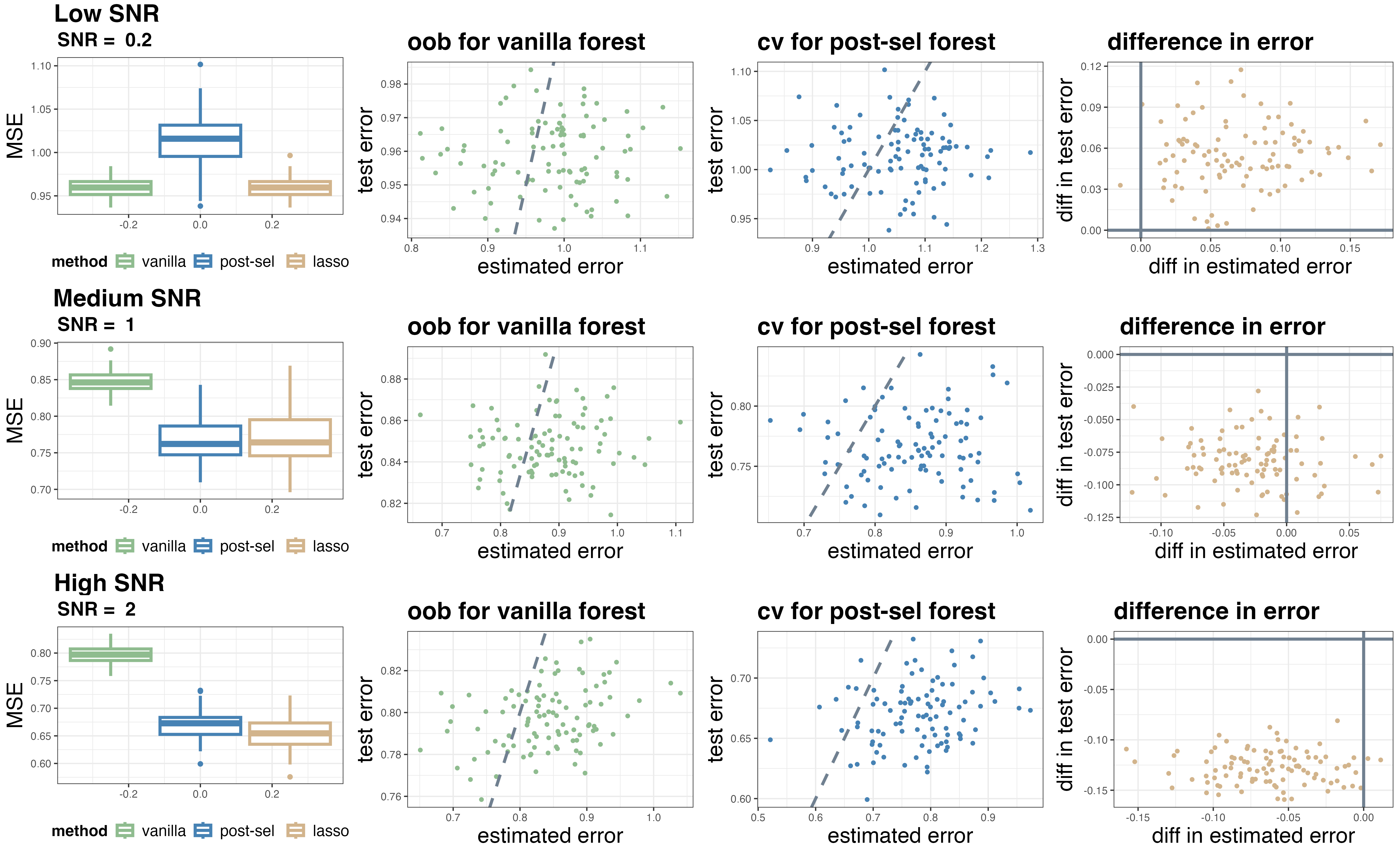}
    \caption{\em{Out-of-bag and Cross-validation Error Estimates for Polynomial Generating Functions:}
    \small Each row represents a signal-to-noise ratio setting. The first panel shows comparison of prediction mean squared error, second and third panels show the true test error against out-of-bag and cross-validation error estimates, last panel displays differences in true test errors against differences in estimated errors.}
    \label{fig:polynomial_error}
\end{figure}

We can observe that the error estimates are roughly unbiased when SNR is low, while cross-validation tend to overestimate when SNR is high. Fortunately, since the post-selection forest is significantly better when SNR is high, the error estimates still correctly predict the sign of difference in errors. However, this causes an issue when the SNR is not high enough, as we see that the overestimation of cross-validation error causes the adaptive method to act as though the vanilla forest is preferred.

\vspace{3mm}
Next, we decompose the prediction mean squared error into bias and variance terms, and provide a deeper understanding of how vanilla and post-selection forest perform differently under different SNRs. \autoref{fig:polynomial_biasvar} shows the decomposition and tracks performance of the Lassoed forest alongside the two benchmarks.
\begin{figure}[H]
    \centering
    \includegraphics[width=1\linewidth]{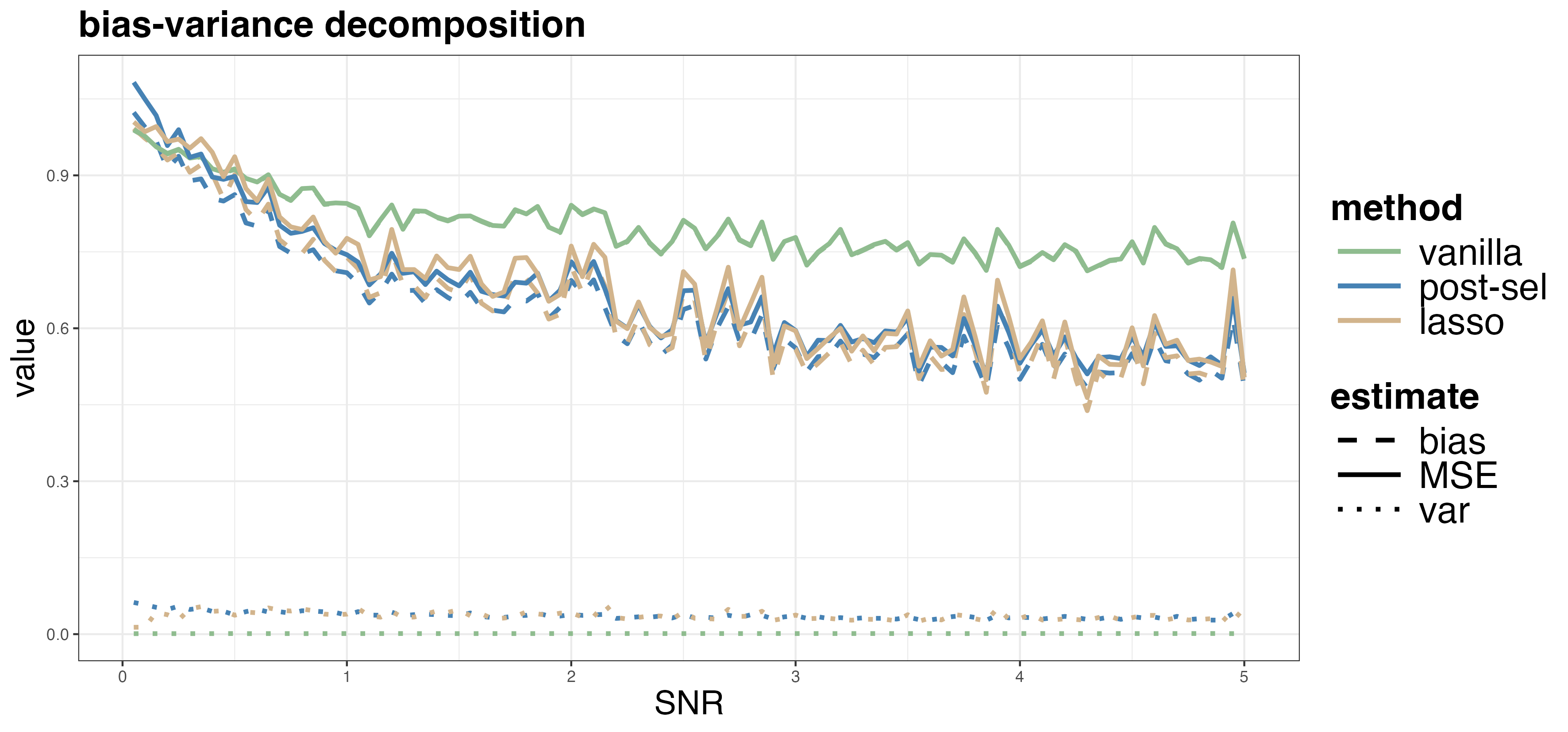}
    \caption{\em{Bias-variance Decomposition for Polynomial Generating Functions}}
    \label{fig:polynomial_biasvar}
\end{figure}

The result justifies our assumptions, as well as verifies our theorems and conjectures:
\begin{enumerate}
    \item We assume the bias of base learners from a given function class is a monotonically non-increasing function of the SNR in \cref{assump:equibias}, for a fixed sample size. This is justified by the decreasing bias of the vanilla forest, as we showed in \cref{rm:tradeoff} that vanilla forest does not reduce bias and therefore always reflects the bias $\eta(s)$ of base learners.
    \item We also assume the same monotone relation between variance of base learner and the SNR in \cref{assump:homoscedasticity}. This is justified by the decreasing variance of post-selection forest, particularly when SNR is small. 
    \item \cref{prop:SNR} is verified, as we see not only the mean squared error decreases monotonically when SNR grows, but the  vanilla and post-selection forest also have different rates of decreasing. Meanwhile, the numerical result illustrates the signal-to-noise ratio dependency.
    \item We also notice that the variance of both vanilla and post-selection forest is on a much smaller scale compared with the bias, so is the change in variance. The empirical evidence also shows that the gap in variance is shrinking, but subtly. This verifies \cref{thm:tradeoff} that both variances have a fixed component $\phi(\Sigma)$ from the data, and the uncertainties from fitting base learners scale differently with SNR.
    \item Lastly, the performance of Lassoed forest verifies \cref{thm:adaptive}, where we show that the reduction in mean squared error is theoretically guaranteed, and does not depend on the absolute magnitude of either bias or variance from vanilla and post-selection forest.
\end{enumerate}

Finally, as an important advantage of the tree-based methods is the interpretability, we also want to compare the ability of all three methods to pick the correct set of predictors when the signal is sparse. \autoref{fig:polynomial_variable} shows the comparison under different SNRs, where only the first 5 predictors have non-zero coefficients that are identical and fixed in data generation process. We consider the counts of variable used at each node split as proxies of variable importance. Formally, the measure of variable importance is defined as follows, whose values are normalized and sum up to 1 for each method:
\begin{equation} \label{eq:variableimp}
    \kappa_s =\theta\sum_{j=1}^J\frac{\hat{\gamma}_j C_{s,j}^{L}}{\sum_{s=1}^p\sum_{j=1}^
    J\hat{\gamma}_j C_{s,j}^{L}} +(1-\theta)\sum_{j=1}^J\frac{C_{s,j}^{V}}{\sum_{s=1}^p\sum_{j=
    1}^JC_{s,j}^{V}}, \quad s=1,...,p
\end{equation}
where $C_{s,j}^{L}$ is the count of how many times the $s^{th}$ variable is used as a splitting node in the $j^{th}$ tree for Lassoed forest, same to $C_{s,j}^{V}$ for vanilla forest. The above definition yields the following implications for different methods:
\vspace{-1mm}
\begin{enumerate}
    \item For vanilla forest, $\kappa_s$ is the count of $s^{th}$ variable as splitting node, divided by the total number of splits;
    \item For post-selection forest, $\kappa_s$ is the count of $s^{th}$ variable as splitting node, weighted by the Lasso coefficient for each tree, and divided by the sum of weighted counts;
    \item For Lassoed forest, $\kappa_s$ is the adaptive weighted average of the above two.
\end{enumerate}
\vspace{-2mm}
\begin{figure}[H]
    \centering
    \includegraphics[width=0.9\linewidth]{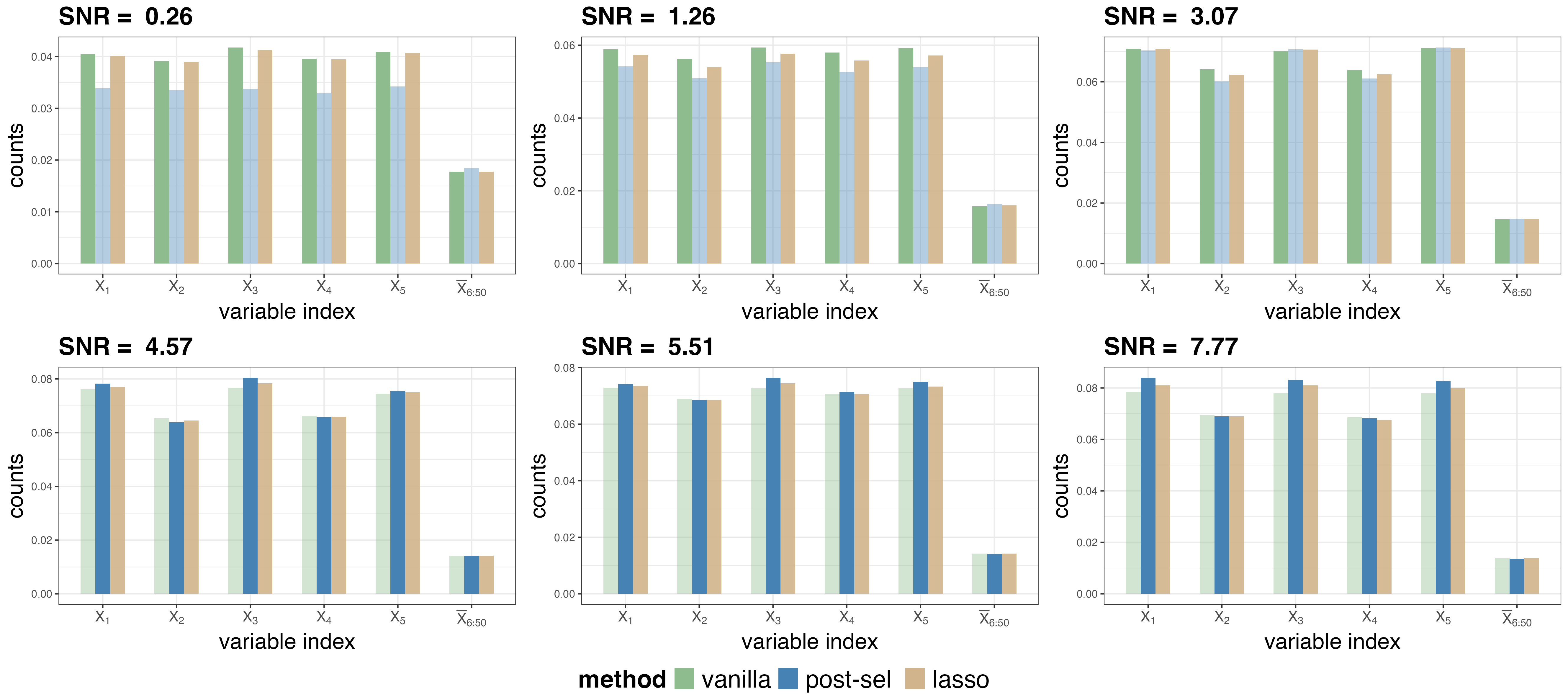}
    \caption{\em{Split Counts Measuring Variable Importance for Polynomial Generating Functions:}
    \small For each signal-to-noise ratio, the split counts of all trees are recorded. For the vanilla forest, we simply divide the counts for each variable by the total number of counts. For the post-selection forest, the proportions are calculated only based on selected trees and counts weighted by the Lasso coefficients. For the Lassoed forest, the weighted averages are used. Better measures are colored in a darker shade.}
    \label{fig:polynomial_variable}
\end{figure}

We can observe that when the signal-to-noise ratio is low, post-selection could not always correctly pick the \textit{good trees} that use true predictors as splitting nodes. As signal-to-noise ratio increases, post-selection does a better job in filtering out trees that are poorly fitted, and selects a small subset of trees that could give a better measure of variable importance than the noisy version given by the vanilla forest. As expected, the Lassoed forest can consistently produce a measure that is admissible, regardless of the signal-to-noise ratio.

\subsection{Tree Ensembles as Input}
To show that the phenomena we observed above are universal, and not unique to the case where data are generated according to simple polynomial functions, now we discuss the case where tree ensembles are used in the data generation process. This is the case where we should expect the vanilla forest to have good performance due to the nature of input data. However, we show below that the Lasso selection step could still further improve the model performance when SNR is high. Moreover, the Lassoed forest again shows its merit in balancing between the two methods and provide an overall good performance across different SNRs.

In particular, we use the following data generating function:
\begin{equation*}
\begin{aligned}
&\begin{cases}
    \hat{T}^0 \leftarrow \text{Tree}(X, y^0)\\
    \hat{T}_1=\hat{T}^0_1(X), \quad \hat{T}_j= \hat{T}^0_j(X)+\rho\cdot\hat{T}_{j-1},\,\forall j\geq2
\end{cases}\\
&\quad y = \sum_{j=1}^p \beta_j \hat{T}_j + \epsilon
\end{aligned}
\end{equation*}
where $y^0\overset{i.i.d.}{\sim} \mathcal{N}(0, I_n)$ are random noises, $\rho = 0.5$ is the correlation between consecutively generated trees, $\mathbf{x}_i\overset{i.i.d.}{\sim} \mathcal{N} (0, I_p)$, $\epsilon \overset{i.i.d.}{\sim} N (0, \sigma^2)$. $\boldsymbol{\beta}$ is again a sparse vector generated by $ \beta_j = b \cdot \text{Unif}(0,c) + (1-b) \cdot 0$, with $b\overset{i.i.d.}{\sim} \text{Bern}(\pi)$, $c=0.1$, $\pi=0.5$.

\vspace{3mm}
As before, we set sample size $n=400$, number of trees $J=200$, number of features $p=50$. We preprocess the data by scaling and centering the response $y$. The adaptive weight  $\theta$ is chosen from a grid of $\{0,0.25,0.5,0.75,1\}$. \autoref{fig:tree_MSE} shows the comparison of mean squared error across different SNRs. \autoref{fig:tree_error} shows the accuracy of error estimates and its impact on the performance of Lassoed forest.

\begin{figure}[t]
    \centering
    \includegraphics[width=1\linewidth]{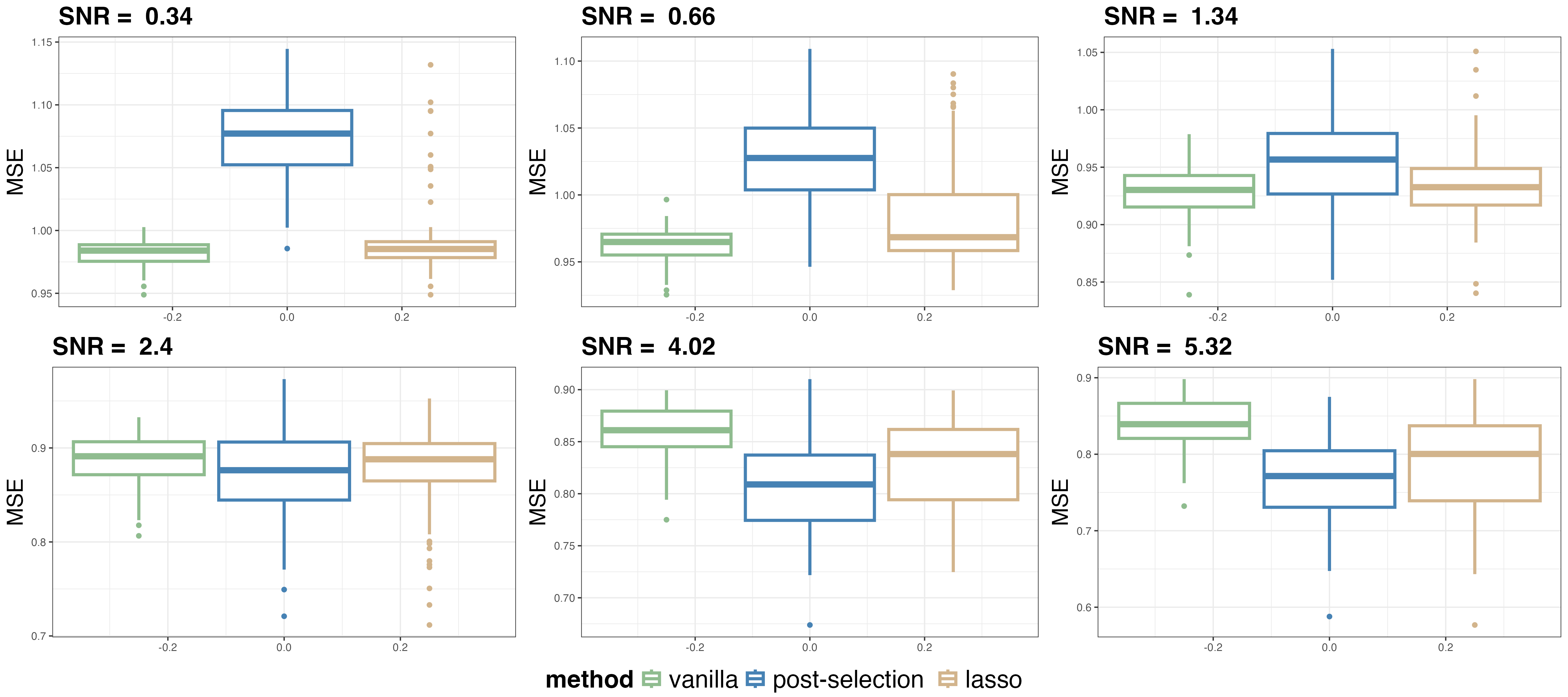}
    \caption{\em{Mean Squared Error for Tree-ensemble Generating Functions}}
    \label{fig:tree_MSE}
\end{figure}

\vspace{3mm}
We have some similar findings for the vanilla and post-selection forest as before, but we also observe something different in terms of the performance of Lassoed forest:
\begin{enumerate}
    \item The post-selection forest is more preferable when the signal-to-noise ratio is relatively high. Applying the Lassoed forest is almost always beneficial, but still comes with a cost when vanilla and post-selection forest have comparable performance, since difference in the noisy error estimates fails to predict the sign of difference in true errors.
    \item We no longer observe a particular SNR setting where the Lassoed forest is strictly better than the other two methods. This can be explained by the smaller gap between the model performance of vanilla and post-selection forest. As discussed in \cref{rm:adaptive}, mean squared error from the adaptive method behaves like a weighted sum of mean squared error from the other two methods, with sum of the weights $\leq1$. However, the numerical results show when the error estimates are noisy and the two benchmark methods have comparable performances, picking the optimal set of weights could be challenging in practice.
\end{enumerate}

\begin{figure}[H]
    \centering
    \includegraphics[width=1\linewidth]{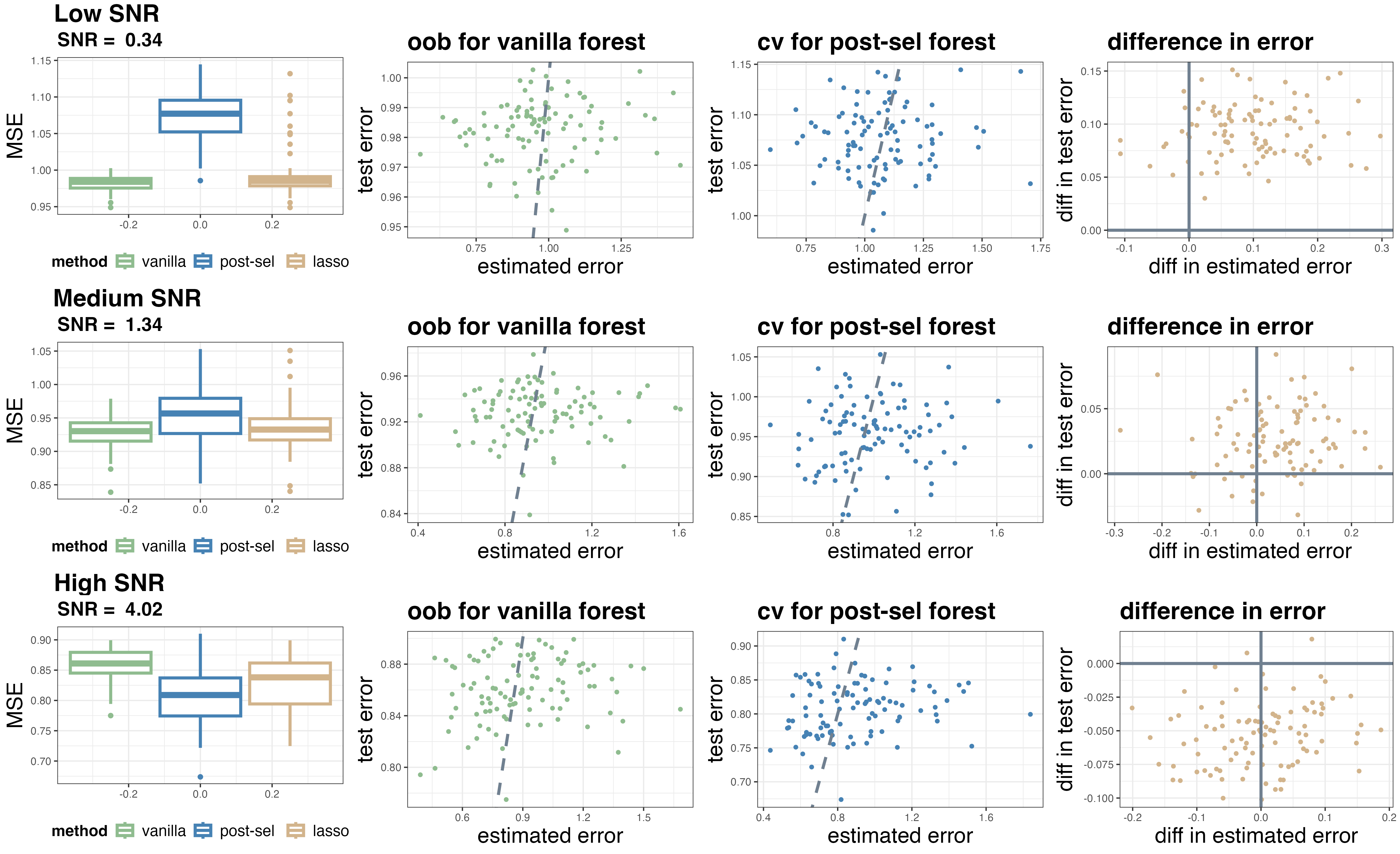}
    \caption{\em{Out-of-bag and Cross-validation Error Estimates for Tree-ensemble Generating Functions}}
    \label{fig:tree_error}
\end{figure}

Lastly, we look at the bias-variance tradeoff when tree ensembles are used as the generating function. \autoref{fig:tree_biasvar} shows the error decomposition of all three methods. As before, we observe the followings:
\begin{enumerate}
    \item The bias of vanilla and post-selection forest scales differently with SNR, where the post-selection forest shows a faster decreasing rate as SNR grows.
    \item The variance of vanilla and post-selection forest, as well as its gap, almost stay constant regardless of SNR. This suggests the variance from data itself is the dominating source of uncertainty, rather than the variance from training base learners.
    \item The variance is small in magnitude for both methods, hence the mean squared error mostly depends on the bias term.
    \item The performance of Lassoed forest is admissible, i.e. at least as good as both vanilla and post-selection forest. However, the ambiguous contrast in the other two methods hinders the adaptive method from picking the optimal weights and attaining further improvements.
\end{enumerate}

\begin{figure}[H]
    \centering
    \includegraphics[width=1\linewidth]{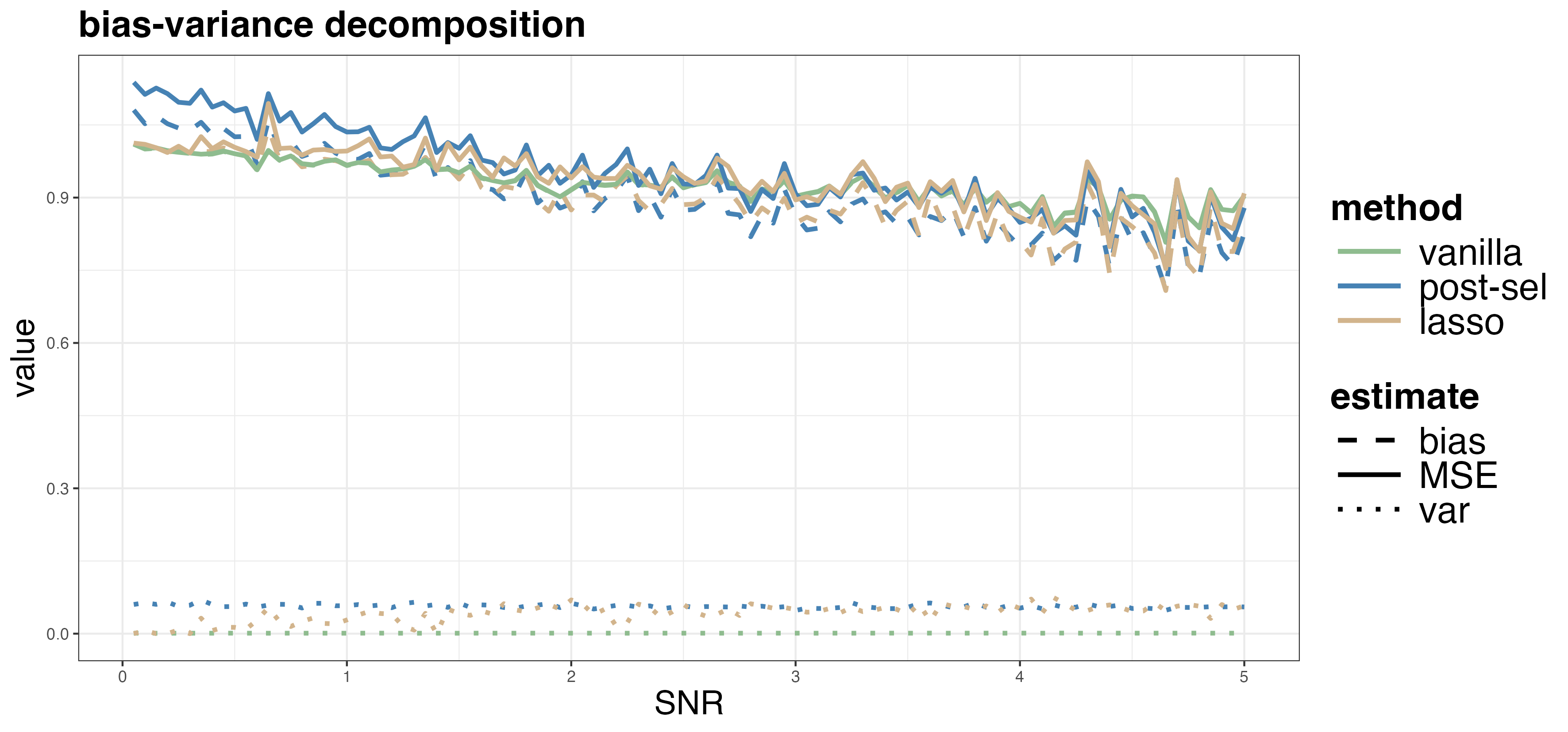}
    \caption{\em{Bias-variance Decomposition for Tree-ensemble Generating Functions}}
    \label{fig:tree_biasvar}
\end{figure}

\section{Experiments} \label{sec:application}
In this section, we provide three examples to demonstrate the utility of our method in real world settings. 
\subsection{Cancer Cell Line Sensitivity Analysis} \label{subsec:CCL}

First, we evaluate the effectiveness of our method on real-world data collected from pre-clinical cancer cell line (CCL) viability screens. We study the regression problem of the sensitivity CCLs to a given drug using CCL gene expression as features. 

\vspace{3mm}
We run our analysis on the dataset processed and provided by \citeasnoun**{Nguyen2023.01.08.522775}. The dataset is a collection of RNA-seq profiles for different sets of cancer cell lines and pharmacological data, from \textit{the Genentech Cell line Screening Initiative} (gCSI) \cite**{haverty2016reproducible} \cite**{klijn2015comprehensive}, \textit{the Cancer Therapeutics Response Portal} (CTRPv2) \cite**{rees2016correlating} \cite**{seashore2015harnessing} \cite**{basu2013interactive}, and \textit{the Genomics of Drug Sensitivity in Cancer} (GDSC2) \cite**{yang2012genomics} \cite**{iorio2016landscape} \cite**{garnett2012systematic}. The dataset includes data of 1,602 unique cell lines, 205 unique drugs, and 233,921 drug sensitivity experiments. For features we use the genes measured in the \textit{the Library of Integrated Network-Based Cellular Signatures} (LINCS) L1000 project \cite**{subramanian2017next}.

\vspace{3mm}
Formally, we take the RNA-seq gene expressions as the feature matrix, where each cell line is one row. The gene expression values are given in units of $\text{log}_2(x_{ij}+0.001)$, where $x_{ij}$ is the expression level of gene $j$ in sample $i$, given in Transcripts per Million (TPM). For each drug from the set $\{$\textit{Vincristine, Paclitaxel, Gemcitabine, Erlotinib, Lapatinib, Vorinostat}$\}$, we take the drug response measured by area above the curve (AAC) as a continuous response variable between 0 and 1, with higher values corresponding to great sensitivity. The choice of drugs is meant to cover targeted drugs and broad-effect chemotherapies in cancer treatment. We first train and test the models on CTRPv2 data, with further validation across different independent datasets provided in \autoref{appendix:CCL}. \autoref{tab:CCL} displays the test error given by different models on a held-out set of the CTRPv2 data, with smaller value indicating better predictive power.

\begin{table}[H]
    \centering
    \includegraphics[width=1\linewidth]{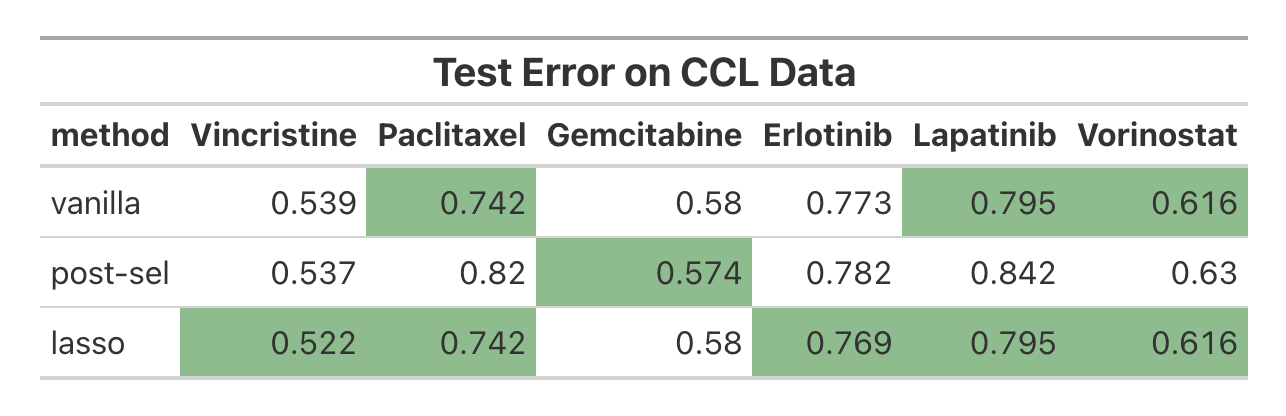}
    \caption{\em{Comparison of Test Error for Different Drugs from CCL Data}}
    \label{tab:CCL}
\end{table}

We can observe the following from the above table:
\begin{enumerate}
    \item The vanilla and post-selection forest have comparable performance for most drugs, with vanilla forest having a slight advantage.
    \item The Lassoed forest always produces the smallest test error, except for on \textit{Gemcitabine}, where it loses by a small margin. There are also cases where the adaptive method outperforms both of its competitors, such as on \textit{Vincristine} and \textit{Erlotinib}. The improvement of test error suggests that the adaptive method can build better predictive models by adaptively implementing post-selection.
\end{enumerate}

A number of parameters in this experiment can be further tuned, such as the number of trees and the maximum number of terminal nodes. However, the results suggest that the Lassoed forest is a promising machine learning tool to predict CCL sensitivity. 

\subsection{Immune Checkpoint Treatment Response Analysis}
Next, we tested our method on the the task of predictive patient response to immune checkpoint inhibitors (ICIs).  We applied our method to two additional types of problems: survival regression and binary classification. In both cases we used gene expression measurements as features and compared our method with the vanilla and post-selection forests as well as standard generalized linear model (GLM) methods.

For our evaluation we used the data from \citeasnoun**{liu2019integrative} which investigated the efficacy of PD-1 blockade in melanoma. We utilized the pre-processed data from the ORCESTRA platform \cite**{mammoliti2021orchestrating}.

For survival regression, we trained our model to predict the progression-free survival time (PFS) using log-rank as the split rule to construct first-stage random forest. We evaluate the method using the concordance index (c-index) \cite**{harrell1982evaluating}. The c-index takes values between 0 and 1, with 1 corresponding to a perfect ordering, 0.5 corresponding to random prediction, and 0 corresponding to a complete inversion. 

For binary classification, we binned RECIST data into two groups, responders (R) (comprising complete response (CR) and partial response (PR) as response) and non-responders (NR) (comprising stable disease (SD) and progressive disease (PD)). We construct probability trees that predict class probability at first stage. We used the classification error rate (1-accuracy) to measure model performance. 

We implement a 10-fold cross-validation and report the average performance metrics. \autoref{tab:immune} shows the comparison of different models, using Cox and binomial GLM as benchmarks respectively.

\begin{table}[H]
    \centering
    \begin{subtable}{0.49\linewidth}
        \centering
        \includegraphics[height=6cm]{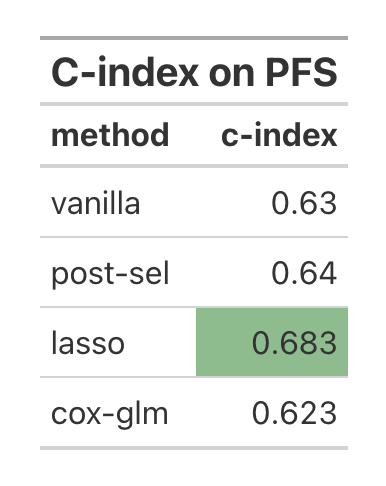}
        \caption{\em{C-index for Survival Regression}}
    \end{subtable}
    \hfill
    \begin{subtable}{0.49\linewidth}
        \centering
        \includegraphics[height=6cm]{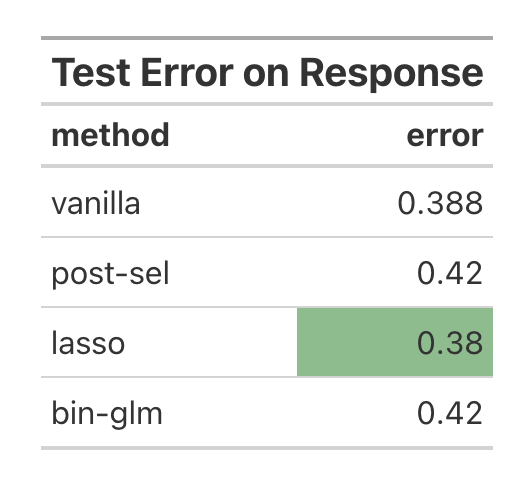}
        \caption{\em{Test Error for Binary Clasification}}
    \end{subtable}
    \caption{\em{Comparison of Model Performance with GLMs}}
    \label{tab:immune}
\end{table}

From the table, we can observe the following:
\begin{enumerate}
    \item Among the three random-forest-based methods, the Lassoed forest has the highest average c-index and lowest mis-classification rate. These results verify that our method can be an improvement over both the vanilla and post-selection forests.
    \item The Lassoed forest outperforms the generalized linear models in both tasks. This suggests that our method is a competitor to "off-the-shelf" parametric models.
\end{enumerate}
The above two analyses on immunotherapy data demonstrate that the Lassoed forest can be readily adapted for different tasks, including survival regression and binary classification. They also show that it provides a systematic way to build non-parametric models with performance improvement.

\subsection{HIV Drug Resistance Analysis}
Last, we test our method on a dataset where the phenotypic drug response $y$ is continuous, but the genotypic predictors $X$ are all binary mutations. The goal is to identify Drug Resistance Mutations (DRMs) associated with Nucleoside RT inhibitors (NRTIs). 

We utilize the data that are publicly available from the Stanford HIV Drug Resistance Database (HIVDB) \cite**{rhee2003human,shafer2006rationale}. The genotype-phenotype correlation dataset contains isolates on which in vitro susceptibility tests were performed using the PhenoSense assay (Monogram, South San Francisco, USA). It includes 12,442 phenotype results from 2,167 isolates obtained from persons infected with HIV-1.

We take amino acid mutation indicators as predictors $X$, log-transformed fold resistance as response $y$, where \textit{fold resistance} is defined as the increase in drug concentration required to inhibit the virus by 50\% (IC50) relative to a wild-type virus. We perform the analysis on six widely used NRTIs: \textit{Lamivudine} (3TC), \textit{Abacavir} (ABC), \textit{Zidovudine} (AZT), \textit{Stavudine} (D4T), \textit{Didanosine} (DDI), \textit{Tenofovir} (TDF). Again, we set the maximum number of terminal nodes to be 5, in order to select potentially important amino acid positions. \autoref{tab:NRTI} shows the comparison of different models in terms of mean squared error on the test set.

\begin{table}[H]
    \centering
    \includegraphics[width=1\linewidth]{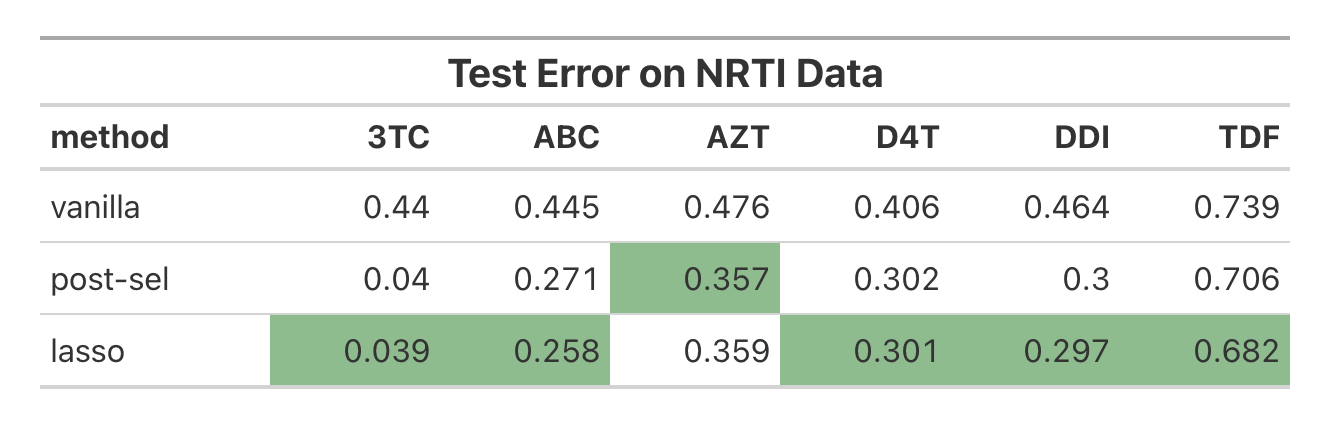}
    \caption{\em{Comparison of Test Error for Different Drugs from NRTI Data}}
    \label{tab:NRTI}
\end{table}

We can observe that:
\begin{enumerate}
    \item In this dataset, the post-selection forest has much better performance in selecting predictive amino acid positions for most of the drugs. In particular, it improved the performance on \textit{Lamivudine} (3TC) by 90\%, \textit{Abacavir} (ABC) by 40\%, \textit{Zidovudine} (AZT) by 25\%, etc.
    \item The Lassoed forest further improves the model, as it almost always gives the best performance among all three methods. This also corresponds to our previous finding that, when there is a greater gap between the performance of vanilla and post-selection forest, the adaptive method tends to be more accurate in picking the optimal weights, which leads to a better mixture model than both others.
\end{enumerate}

The above three experiments, along with our motivating examples in \cref{subsec:examples}, show that the approach of adaptively applying Lasso post-selection can be applied to a variety of real world datasets. They also show that the method is capable of serving multiple purposes, including regression, classification, survival analysis, etc. It not only provides a guard against the volatile performance of vanilla and post-selection forest, but also makes further improvement in terms of predictive power compared with other parametric models.

\section{Discussion} \label{sec:discussion}
In this paper, we proposed an adaptive method that combines random forest and Lasso post-selection. We showed that it could outperform both vanilla and post-selection forest in theory, and verified its performance through simulation. We also demonstrated that it could be a useful tool in variable selection by applications on two real-world biomedical datasets.

\vspace{3mm}
Here are some potential directions for further discussion:
\begin{enumerate}
    \item In \cref{sec:theory}, we mainly focused on base learners that have the same bias and variance and are universal approximators. Random forests are just one such learner. We expect the same phenomena would also be detected if either of the assumptions on bias or variance is loosened, and the theorems are readily generalizable. However, more efforts will likely be required if we remove the expressivity assumption. Removing or relaxing the expressivity assumption corresponds to the case where the base learners are under-parametrized, relative to the true relationship between $X$ and $y$.
    \item Also in \cref{sec:theory}, we showed the bias-variance decomposition in the special case where Lasso reduces to least squares. We provided reasonable conjecture when the additional $l_1$ penalty is implemented. We expect the direct derivation to quantify the changes in bias and variance to be hard due to the lack of closed form solutions of Lasso regression. However, we believe that some asymptotic properties of the Lasso adaptive method could merit further investigation.
    \item As shown in \cref{sec:simulation}, the uncertainty of error estimates often hurts the model's ability to pick the optimal set of weights, and in turn the overall performance. Therefore, we are actively looking for alternatives to out-of-bag and cross-validation error estimates, or improvements to these estimates.
    \item In \cref{sec:simulation} and \cref{sec:application}, we showed the potential of using tree-based methods to perform variable selection. We simply used the split counts of entire forest, or of selected trees, as a measure of variable importance. A modest variation could be to apply the Lasso coefficients on selected trees, and calculate the variable importance based on weighted split counts. The performance of such variation needs to be further investigated.
\end{enumerate}

{\bf Acknowledgements.} R.T. was supported by the NIH (5R01EB001988-
16) and the NSF (19DMS1208164).

\bibliography{reference}
\bibliographystyle{agsm}

\appendix
\section{Proof of \Cref{prop:SNR}}\label{appendix:SNR}
\input{appendix/SNR}

\section{Proof of \Cref{thm:tradeoff}}\label{appendix:tradeoff}
\input{appendix/tradeoff}

\section{Proof of \Cref{thm:adaptive}}\label{appendix:adaptive}
\input{appendix/offset}

\section{Variance and Covariance of Base Learners}\label{appendix:variance}
\input{appendix/variance}

\section{Method Validation on Independent CCL Screens} \label{appendix:CCL}
\input{appendix/CCL}
\end{document}

%% file: appendix/SNR.tex
We investigate the bias and variance terms for $\hat{y}_i^{mean}$ and $\hat{y}_i^{reg}$ respectively.

\vspace{3mm}
For $\hat{y}_i^{mean}$, 
\begin{equation*}
\begin{aligned}
    \mathbb{E}[\hat{y}^{mean}] = \mathbb{E}\left[\frac{1}{J}\sum_{j=1}^Jf_j(\mathbf{x}^*)\right]
    =\frac{1}{J}\sum_{j=1}^J(\mu + \eta(s))
&= \mu + \eta(s)
\end{aligned}
\end{equation*}
where the second equality holds by \cref{assump:equibias} on equibias.

\begin{equation*}
\begin{aligned}
    \text{Var}(\hat{y}^{mean}) &= \text{Var}\left(\frac{1}{J}\sum_{j=1}^Jf_j(\mathbf{x}^*)\right)\\
    &= \mathbb{E}\left[ \text{Var}\left(\frac{1}{J}\sum_{j=1}^Jf_j(\mathbf{x}^*)\Big|\mathbf{x}^*\right)\right] + \text{Var}\left(\mathbb{E}\left[\frac{1}{J}\sum_{j=1}^Jf_j(\mathbf{x}^*)\Big|\mathbf{x}^*\right]\right)\\
    &=\frac{1}{J}\mathbb{E}\left[\tau^2(\mathbf{x}^*, s)\right]+ \frac{J(J-1)}{J^2}\mathbb{E} \left[\rho(\mathbf{x}^*, s)\right] + \text{Var}\left(g(\mathbf{x}^*)+\eta(s)\right)\\
    &=\frac{1}{J}\psi(s) +\frac{J(J-1)}{J^2}\omega(s)+ \phi(\Sigma)\\
\end{aligned}
\end{equation*}
where the second equality holds by law of conditional variance, third and last equality hold by \cref{assump:equibias,assump:homoscedasticity} on equibias and homoscedasticity.

\vspace{3mm}
For $\hat{y}^{reg}$, first we assume the oracle $\gamma_j$ in \eqref{eq:est_reg} are known,

\begin{equation*}
\begin{aligned}
    \mathbb{E}[\hat{y}^{reg}] 
    = \mathbb{E}\left[\gamma_0 +\sum_{j=1}^J \gamma_j f_j(\mathbf{x}^*)\right] = -\eta(s) + \sum_{j=1}^J \gamma_j[\mu +\eta(s)] = \mu
\end{aligned}
\end{equation*}
where the second and third equality hold by \cref{assump:expressivity} on expressivity.

\begin{equation*}
\begin{aligned}
    \text{Var}(\hat{y}^{reg}) &=\text{Var}\left(\gamma_0 +\sum_{j=1}^J \gamma_j f_j(\mathbf{x}^*)\right)\\
    &=\mathbb{E}\left[ \text{Var}\left(\sum_{j=1}^J \gamma_j f(\mathbf{x}^*)\Big|\mathbf{x}^*\right)\right] + \text{Var}\left(\mathbb{E}\left[\sum_{j=1}^J \gamma_j f_j(\mathbf{x}^*)\Big|\mathbf{x}^*\right]\right)\\
    &=\sum_{j=1}^J \gamma_j^2\mathbb{E}\left[\tau^2(\mathbf{x}^*, s)\right] + 2\sum_{j<k}^J\gamma_j\gamma_k\mathbb{E}\left[\rho(\mathbf{x}^*, s)\right]+\text{Var}\left(g(\mathbf{x}^*)+\eta(s)\right)\\
    &=\sum_{j=1}^J \gamma_j^2\psi(s) + 2\sum_{j<k}^J\gamma_j\gamma_k \omega(s)+ \phi(\Sigma)\\
    &=\phi(\Sigma)
\end{aligned}
\end{equation*}
where the third and last equality hold by \cref{assump:equibias,assump:homoscedasticity} on equibias and homoscedasticity. The last equality holds by \cref{assump:expressivity} on expressivity, which implies $\sum_{j=1}^J \gamma_j^2\psi(s) + 2\sum_{j<k}^J\gamma_j\gamma_k \omega(s)=0$.

\vspace{3mm}
Now we consider the case where $\gamma_j$ are estimated by $\hat{\gamma}_j$. We assume that under a fixed sample size $N$, $\hat{\gamma}_j$ has conditional expectation $\mathbb{E}[\hat{\gamma}_j|\mathcal{F}] = \gamma_j+M_j(s)$, where $\mathcal{F}= (f_1,...f_J)$ are pre-trained base learners. We also assume that $M_j(s)$ is non-increasing in $s$ and non-negative, $j=0,1,...J$. Now the bias and variance terms become:

\begin{equation*}
\begin{aligned}
    \mathbb{E}[\hat{y}^{reg}] 
    &= \mathbb{E}\left[\mathbb{E}\left[\hat{\gamma}_0 +\sum_{j=1}^J \hat{\gamma}_j f_j(\mathbf{x}^*)\Bigg|\mathcal{F}\right]\right]\\
    &= \mathbb{E}\left[\mathbb{E}[\hat{\gamma}_0|\mathcal{F}]+\sum_{j=1}^J \left(\text{Cov}(\hat{\gamma}_j,f_j(\mathbf{x}^*)|\mathcal{F})+\mathbb{E}\left[\hat{\gamma}_j|\mathcal{F}\right]\mathbb{E}\left[f_j(\mathbf{x}^*)|\mathcal{F}\right]\right)\right]\\
    &= \mathbb{E}\left[\gamma_0 + M_0(s)+\sum_{j=1}^J \mathbb{E}\left[(\gamma_j + M_j(s))f_j(\mathbf{x}^*)|\mathcal{F}\right]\right]\\
    &= \mu + M_0(s) + \mu\sum_{j=1}^J M_j(s)
\end{aligned}
\end{equation*}
where the third equality holds because of the independency of $\mathbf{x}^*$, third equality holds because $g(\mathbf{x}^*) =\gamma_0+\sum_{j=1}^J \gamma_j f_j(\mathbf{x}^*)$ by \cref{assump:expressivity} on expressivity. This shows that $\hat{y}^{reg}$ is no longer unbiased as before. We can also observe that the bias is non-increasing in SNR, and non-negative by assumption. 

\vspace{3mm}
To analyze the variance of $\hat{y}^{reg}$, we assume that a under fixed sample size $N$, $\hat{\gamma}_j$ has conditional variance $\text{Var}(\hat{\gamma}_j|\mathcal{F}) = V_j(s)$, conditional covariance $\text{Cov}(\hat{\gamma}_j, \hat{\gamma}_k|\mathcal{F}) = V_{jk}(s)$, where $\mathcal{F}= (f_1,...f_J)$ are pre-trained base learners. We further assume that $V_j(s)$, $V_{jk}(s)$ are non-increasing in $s$ and non-negative,  $j=0,1,...J$, $j<k$.
\begin{align*}
    \text{Var}(\hat{y}^{reg}) =&\text{Var}\left(\hat{\gamma}_0 +\sum_{j=1}^J \hat{\gamma}_j f_j(\mathbf{x}^*)\right)\\
    =&\mathbb{E}\left[ \text{Var}\left(\hat{\gamma}_0 +\sum_{j=1}^J \hat{\gamma}_j f_j(\mathbf{x}^*)\Big|\mathcal{F}\right)\right] + \text{Var}\left(\mathbb{E}\left[\hat{\gamma}_0 +\sum_{j=1}^J \hat{\gamma}_j f_j(\mathbf{x}^*)\Big|\mathcal{F}\right]\right)\\
    =&\mathbb{E}\left[ \text{Var}(\hat{\gamma}_0|\mathcal{F}) + \sum_{j=1}^J f_j(\mathbf{x}^*)^2 \text{Var}\left(\hat{\gamma}_j |\mathcal{F}\right)\right.\\ &\left.+ 2\sum_{j=1}^J f_j(\mathbf{x}^*)\text{Cov}(\hat{\gamma}_0,\hat{\gamma}_j |\mathcal{F}) + 2\sum_{1\leq j<k}^J f_j(\mathbf{x}^*)f_k(\mathbf{x}^*)\text{Cov}(\hat{\gamma}_j ,\hat{\gamma}_k |\mathcal{F})\right]\\
    &+ \text{Var}\left(\mathbb{E}\left[\hat{\gamma}_0|\mathcal{F}\right] + \sum_{j=1}^J(\text{Cov}(\hat{\gamma}_j,f_j(\mathbf{x}^*)|\mathcal{F})+\mathbb{E}\left[\hat{\gamma}_j|\mathcal{F}\right]\mathbb{E}\left[f_j(\mathbf{x}^*)|\mathcal{F}\right])\right)\\
    =&\mathbb{E}\left[V_0(s) + \sum_{j=1}^J f_j(\mathbf{x}^*)^2 V_j(s)+ 2\sum_{j=1}^J f_j(\mathbf{x}^*)V_{0j}(s) + 2\sum_{1\leq j<k}^J f_j(\mathbf{x}^*)f_k(\mathbf{x}^*)V_{jk}(s)\right] \\
    &\displaybreak[1]+ \text{Var}\left( \gamma_0 + M_0(s) + \sum_{j=1}^J\text{Cov}(\hat{\gamma}_j,f_j(\mathbf{x}^*)|\mathcal{F})+\sum_{j=1}^J \mathbb{E}\left[(\gamma_j + M_j(s)) f_j(\mathbf{x}^*)|\mathcal{F}\right]\right)\\
    =&V_0(s) + \sum_{j=1}^J (\psi(s)+\phi(\Sigma)+(\mu+\eta(s))^2) V_j(s)+ 2\sum_{j=1}^J(\mu+\eta(s))V_{0j}(s)\\
    &+ 2\sum_{1\leq j<k}^J (\omega(s)+\phi(\Sigma)+(\mu+\eta(s))^2)V_{jk}(s) + \text{Var}\left(\sum_{j=1}^J M_j(s) \mathbb{E} \left[f_j(\mathbf{x}^*)|\mathcal{F}\right]\right)
\end{align*}
where the last equality holds because of (i) independency of $\mathbf{x}^*$; (ii) $g(\mathbf{x}^*) = \gamma_0+\sum_{j=1}^J \gamma_j f_j(\mathbf{x}^*)$ by \cref{assump:expressivity}; (iii) $\mathbb{E}[f_j(\mathbf{x}^*)f_k(\mathbf{x}^*)]$ can be computed as follows:
\begin{equation*}
    \begin{aligned}
    \mathbb{E}[f_j(\mathbf{x}^*)f_k(\mathbf{x}^*)]
    &= \text{Cov}(f_j(\mathbf{x}^*),f_k(\mathbf{x}^*)) +\mathbb{E}[f_j(\mathbf{x}^*)]\mathbb{E}[f_k(\mathbf{x}^*)]\\
    &= \mathbb{E}[\text{Cov}(f_j(\mathbf{x}^*),f_k(\mathbf{x}^*)|\mathbf{x}^*)] + \text{Cov}(\mathbb{E}[f_j(\mathbf{x}^*)|\mathbf{x}^*],\mathbb{E}[f_k(\mathbf{x}^*)|\mathbf{x}^*])+(\mu+\eta(s))^2\\
    &= \omega(s)+ \phi(\Sigma)+(\mu+\eta(s))^2
    \end{aligned}
\end{equation*}
and same applies to variance. This shows that the variance is also a monotonically non-increasing function of the SNR.

In summary, with additional assumptions on the monotonicity of conditional expectation and covariance, we showed the mean squared error of both vanilla and Lasso forest are monotonically non-increasing in the signal-to-noise ratio. While the result is not surprising, it serves as a foundation of our proceeding analysis of the bias-variance decomposition. It also helps understanding of the impact of applying Lasso coefficients, and supports our conjecture in \cref{subsec:tradeoff} and \cref{subsec:offset}.

%% file: appendix/tradeoff.tex
We analyze the bias and variance of $\hat{y}^{reg}$ by first conditioning on all the base learner predictions $\mathcal{P} = \left\{\left(f_1(\mathbf{x}_i), ...f_J(\mathbf{x}_i)\right)\right\}_{i=1}^N$ on $\mathcal{D}$. Denote the base learner prediction matrix for $\{\mathbf{x}_i\}_{i=1}^N$ by $F = \left(1, f_1(\mathbf{x}), ...f_J(\mathbf{x})\right)\in \mathbb{R}^{N\times(J+1)}$, as well as a vector for $\mathbf{x}^*$  by $F^* = \left(1, f_1(\mathbf{x}^*), ...f_J(\mathbf{x}^*)\right)\in \mathbb{R}^{J+1}$. We first consider the case where $J+1<N$:
\begin{equation*}
    \begin{aligned}
    \mathbb{E}[\hat{y}^{reg}] &= \mathbb{E}\left[\mathbb{E}[F^*(F^TF)^{-1}F^Ty|\mathcal{P}, \mathbf{x}^*]\right]\\
    &= \mathbb{E}\left[F^*\mathbb{E}[(F^TF)^{-1}F^T(F\Gamma+\epsilon)|\mathcal{P}]\right]\\
    &=\mathbb{E}[F^*\Gamma]\\
    & = \mu
    \end{aligned}
\end{equation*}
where the second equality holds by independence of $\mathbf{x}^*$ and \cref{assump:expressivity}.

\vspace{3mm}
For variance calculation, we first make assumptions on Gaussianity of the conditional base learner predictions. Without loss of generality, we assume $(f_1(\mathbf{x}),...f_J(\mathbf{x}))|\mathcal{F} \sim \mathcal{N}(\mathbf{0}, \mathbf{W})$, where $\mathcal{F}= (f_1,...f_J)$ are pre-trained base learners, $\mathbf{x} \perp\!\!\!\perp  \mathcal{F}$. Note that if the conditional mean is not $\mathbf{0}$, one can always recenter without impacting the SNR by \cref{rm:SNR}. The conditional covariance matrix $\mathbf{W}$ is defined as follows: let conditional variance $\text{Var}(f_j(\mathbf{x})|\mathcal{F}) = W_{jj}(s)$, conditional covariance $\text{Cov}(f_j(\mathbf{x}), f_k(\mathbf{x})|\mathcal{F}) = W_{jk}(s)$, $\mathbf{W}=\mathbf{W}(s) = [W_{jk}(s)]_{1\leq j,k\leq J}$.
\begin{align*}
    \text{Var}(\hat{y}^{reg}) &=\mathbb{E}\left[\text{Var}(F^*(F^TF)^{-1}F^Ty|\mathcal{P},\mathbf{x}^*)\right] + \text{Var}\left(\mathbb{E}[F^*(F^TF)^{-1}F^Ty|\mathcal{P},\mathbf{x}^*]\right) \\
    &=\mathbb{E}\left[\text{Var}(F^*(F^TF)^{-1}F^T(F\Gamma+\epsilon)|\mathcal{P},\mathbf{x}^*)\right] + \text{Var}\left(\mathbb{E}[F^*(F^TF)^{-1}F^T(F\Gamma+\epsilon)|\mathcal{P},\mathbf{x}^*]\right) \\
    &\displaybreak[1]=\mathbb{E}\left[F^*(F^TF)^{-1}F^T\text{Var}(\epsilon)F(F^TF)^{-1}F^{*T}\right] + \text{Var}\left(\gamma_0 + \sum_{j=1}^J \gamma_j f_j(\mathbf{x}^*)\right)\\
    &=\sigma^2 \mathbb{E}\left[F^*(F^TF)^{-1}F^{*T}\right]+ \text{Var}(g(\mathbf{x}^*))\\
    &=\sigma^2 \mathbb{E}\left[\mathbb{E}\left[\text{Tr}\left(F^{*T}F^*(F^TF)^{-1}\right)|\mathcal{F}\right]\right]+ \phi(\Sigma)\\
    &=\sigma^2 \text{Tr}\left(\mathbb{E}\left[\mathbb{E}\left[F^{*T}F^*|\mathcal{F}\right]\mathbb{E}\left[(F^TF)^{-1}|\mathcal{F}\right]\right]\right)+ \phi(\Sigma)
\end{align*}
where the last equality holds by independence of $\mathbf{x}^*$. By the assumption of Gaussianity, now $F^{*T}F^*|\mathcal{F}$ follows Wishart distribution $\mathcal{W}_J(\mathbf{W}, N)$, while $(F^TF)^{-1}|\mathcal{F}$ follows inverse-Wishart distribution $\mathcal{W}^{-1}_J(\mathbf{W}^{-1}, N)$. Hence, their moments are as follows:
\begin{equation*}
    \begin{aligned}    \mathbb{E}\left[F^{*T}F^*|\mathcal{F}\right]
    &= \mathbf{W}\\
    \mathbb{E}\left[(F^TF)^{-1}|\mathcal{F}\right]&= \frac{\mathbf{W}^{-1}}{N-J-1}    \end{aligned}
\end{equation*}
which gives $\text{Var}(\hat{y}^{reg}) = \frac{\sigma^2}{N-J-1}+ \phi(\Sigma)$.

\vspace{3mm}
Now we consider the case where $J+1>N$, which suggests the OLS estimator to be in form of the minimum-norm interpolator \cite**{mallinar2024minimum} $\hat{\Gamma} = F^T(FF^T)^{-1}y$:
\begin{equation*}
    \begin{aligned}
    \mathbb{E}[\hat{y}^{reg}] &= \mathbb{E}\left[\mathbb{E}[F^*F^T(FF^T)^{-1}y|\mathcal{P}, \mathbf{x}^*]\right]\\
    &=\mathbb{E}[F^*F^T(FF^T)^{-1}F]\Gamma\\
    \end{aligned}
\end{equation*}
which suggests that the OLS estimator is no longer unbiased as before.

\vspace{3mm}
Similarly to the previous case, we compute variance as follows:
\begin{equation*}
    \begin{aligned}
    \text{Var}(\hat{y}^{reg}) =&\mathbb{E}\left[\text{Var}(F^*F^T(FF^T)^{-1}y|\mathcal{P},\mathbf{x}^*)\right] + \text{Var}\left(\mathbb{E}[F^*F^T(FF^T)^{-1}y|\mathcal{P},\mathbf{x}^*]\right) \\
    =&\mathbb{E}\left[\text{Var}(F^*F^T(FF^T)^{-1}(F\Gamma+\epsilon)|\mathcal{P},\mathbf{x}^*)\right] + \text{Var}\left(\mathbb{E}[F^*F^T(FF^T)^{-1}(F\Gamma+\epsilon)|\mathcal{P},\mathbf{x}^*]\right) \\
    =&\sigma^2 \mathbb{E}\left[F^*F^T(FF^T)^{-1}(FF^T)^{-1}FF^{*T}\right]+ \Gamma^T\text{Var}(F^*F^T(FF^T)^{-1}F)\Gamma
    \end{aligned}
\end{equation*}

The above derivations show that when $J>N$, the regression post-weighting estimator $\hat{y}^{reg}$ undergoes changes in its behavior: no longer unbiased, and variance of linear regression is known to be large due to overfitting. Even though analyzing the exact form of bias and variance separately is hard, due to the form of minimum-norm interpolator, we still can directly compute the mean squared error and analyze its asymptotics with some additional assumptions. Assume $\gamma_j \overset{i.i.d.}{\sim} \mathcal{N}(0,1/J)$, $\mathbf{W} = \mathbf{I}_J$.

\begin{align*}
    \mathbb{E}\left[(y^*-\hat{y}^{reg})^2\right]=& \mathbb{E}\left[(F^*\Gamma-F^*F^T(FF^T)^{-1}(F\Gamma+\epsilon))^2\right]+\sigma^2\\
    =&\mathbb{E}\left[||F^*(I-F^T(FF^T)^{-1}F)\Gamma||_2^2\right]+\mathbb{E}\left[||F^*F^T(FF^T)^{-1}\epsilon||_2^2\right]+\sigma^2\\
    =&\frac{1}{J}\mathbb{E}\left[F^*(I-F^T(FF^T)^{-1}F)F^{*T}\right]+\sigma^2\mathbb{E}\left[F^*F^T(FF^T)^{-1}(FF^T)^{-1}FF^{*T}\right]+\sigma^2\\
    =&\frac{1}{J}\text{Tr}(\mathbb{E}\left[\mathbb{E}\left[F^{*T}F^*|\mathcal{F}\right]\mathbb{E}\left[(I-F^T(FF^T)^{-1}F)|\mathcal{F}\right]\right])\\
    &+\sigma^2\text{Tr}(\mathbb{E}\left[F^{*T}F^*|\mathcal{F}\right]\mathbb{E}\left[F^T(FF^T)^{-1}(FF^T)^{-1}F|\mathcal{F}\right])+\sigma^2\\
    =& \frac{J-N}{J} + \sigma^2 \left(\frac{N}{J-N}+1\right)\\
    \lim_{\substack{J,N \rightarrow \infty \\ J/N \rightarrow r}}\mathbb{E}\left[(y^*-\hat{y}^{reg})^2\right]=& \lim_{\substack{J,N \rightarrow \infty \\ J/N \rightarrow r}}\frac{J-N}{J} + \sigma^2\frac{J}{J-N} = 1 - \frac{1}{r} + \sigma^2 \frac{1}{1-1/r},\quad r>1
\end{align*}
where the second equality holds as $I-F^T(FF^T)^{-1}F$ is orthogonal to $F^T(FF^T)^{-1}$, the last equality holds by applying the trace trick and moment of the inverse-Wishart distribution, as $(FF^T)^{-1}|\mathcal{F} \sim \mathcal{W}^{-1}_N(I_N,J+1)$. The result is in line with \citeasnoun**{hastie2022surprises} in the case where the number of predictors exceeds the number of samples, except that we also account for the randomness $\epsilon^*$ in $y^*$.

\vspace{3mm}
In summary, we showed that the prediction mean squared error for random forest with regression post-weighting is as follows:
\begin{equation*}
    \mathbb{E}\left[(y^*-\hat{y}^{reg})^2\right]=
    \begin{cases}
    \frac{\sigma^2}{N-J-1}+ \phi(\Sigma),&\text{if } J+1<N\\
    \frac{J-N}{J} + \sigma^2\frac{J}{J-N}, &\text{if } J>N
    \end{cases}
\end{equation*}
which suggests that if we do not add an $l_1$-regularization to the regression, then the performance of post-weighting is sensitive to the relative counts of trees and samples. In particular, when the number of trees exceeds the number of samples, then the regression post-weighting on pre-trained trees could fail to generalize to a new data point.

%% file: appendix/offset.tex
We investigate the bias and variance term for $\hat{y}^{ada}$, in the case where $\lambda = 0$.
\begin{equation*}
    \mathbb{E}[\hat{y}^{ada}] = (1-\theta) \mathbb{E}[\hat{y}^{mean}] + \theta\mathbb{E}[\hat{y}^{reg}]=(1-\theta)(\mu+\eta(s))+\theta\mu=\mu+(1-\theta)\eta(s)
\end{equation*}
which is unbiased if and only if $\theta = 1$.

\vspace{3mm}
Now let tree models trained on first half of data $\mathcal{T} = (\widehat{T}_1,...\widehat{T}_J)$, tree predictions on second half of data $\widehat{\mathbf{T}}= (\widehat{\mathbf{T}}_1,...\widehat{\mathbf{T}}_N)^T=(1, \widehat{T}_{\cdot 1},...\widehat{T}_{\cdot J}) \in \mathbb{R}^{N/2 \times (J+1)}$, on a new given point $\widehat{T}^*= (1,\widehat{T}^*_1,...\widehat{T}^*_J)$, average tree prediction $\overline{T^*} =\frac{1}{J}\sum_{j=1}^J \widehat{T}^*_j$.
\begin{equation*}
    \begin{aligned}
    \text{Var}(\hat{y}^{ada}) 
    =&\text{Var}\left((1-\theta) \overline{T^*} +\theta\widehat{T}^*(\widehat{\mathbf{T}}^T\widehat{\mathbf{T}})^{-1}\widehat{\mathbf{T}}^Ty\right)\\
    =&\mathbb{E}\left[\text{Var}\left((1-\theta) \overline{T^*}+\theta g(\mathbf{x}^*) + \theta\widehat{T}^*(\widehat{\mathbf{T}}^T\widehat{\mathbf{T}})^{-1}\widehat{\mathbf{T}}^T\epsilon\Big|\mathbf{x}^*\right)\right]\\ 
    &+ \text{Var}\left( \mathbb{E}\left[(1-\theta) \overline{T^*} +\theta g(\mathbf{x}^*) + \theta\widehat{T}^*(\widehat{\mathbf{T}}^T\widehat{\mathbf{T}})^{-1}\widehat{\mathbf{T}}^T\epsilon\Big|\mathbf{x}^*\right]\right) \\
    =&\mathbb{E}\left[\text{Var}\left((1-\theta) \overline{T^*} + \theta\widehat{T}^*(\widehat{\mathbf{T}}^T\widehat{\mathbf{T}})^{-1}\widehat{\mathbf{T}}^T\epsilon\Big|\mathbf{x}^*\right)\right]+ \text{Var}\left(g(\mathbf{x}^*)\right)\\
    =&\mathbb{E}\left[(1-\theta) ^2\text{Var}\left(\overline{T^*}\Big|\mathbf{x}^*\right) + \theta^2\text{Var}\left(\widehat{T}^*(\widehat{\mathbf{T}}^T\widehat{\mathbf{T}})^{-1}\widehat{\mathbf{T}}^T\epsilon\Big|\mathbf{x}^*\right)\right.\\
    &+\left. 2\theta(1-\theta)\text{Cov}\left(\overline{T^*}, \widehat{T}^*(\widehat{\mathbf{T}}^T\widehat{\mathbf{T}})^{-1}\widehat{\mathbf{T}}^T\epsilon\Big|\mathbf{x}^*\right)\right]+ \text{Var}\left(g(\mathbf{x}^*)\right)\\
    =&(1-\theta)^2\mathbb{E}\left[\text{Var}\left(\overline{T^*}\Big|\mathbf{x}^*\right)\right]+\theta^2\mathbb{E}\left[\mathbb{E}\left[\widehat{T}^*(\widehat{\mathbf{T}}^T\widehat{\mathbf{T}})^{-1}\widehat{\mathbf{T}}^T\epsilon\epsilon^T\widehat{\mathbf{T}}(\widehat{\mathbf{T}}^T\widehat{\mathbf{T}})^{-1}\widehat{T}^{*T}\Big|\mathbf{x}^*\right]\right]
    + \text{Var} \left(g(\mathbf{x}^*)\right)\\
    =&(1-\theta)^2\mathbb{E}\left[\text{Var}\left(\overline{T^*}\Big|\mathbf{x}^*\right)\right]+\theta^2\sigma^2\text{Tr}\left(\mathbb{E}\left[\mathbb{E}\left[\widehat{T}^{*T}\widehat{T}^*\Big|\mathcal{T}\right]\mathbb{E}\left[(\widehat{\mathbf{T}}^T\widehat{\mathbf{T}})^{-1}\Big|\mathcal{T}\right]\right]\right)
    + \text{Var} \left(g(\mathbf{x}^*)\right)\\
    =&(1-\theta)^2\left(\frac{1}{J}\psi(s) +\frac{J(J-1)}{J^2}\omega(s)\right)+\theta^2\frac{\sigma^2}{N-J-1}+ \phi(\Sigma)
    \end{aligned}
\end{equation*}
where we first condition on the independent new data point $\mathbf{x}^*$, then condition on the fitted tree models $\mathcal{T}$.

%% file: appendix/variance.tex
We derive the variance $\psi(s)$ and covariance $\omega(s)$ of base learners defined in \cref{assump:homoscedasticity} under different modeling assumptions, to facilitate the direct comparison of prediction mean squared errors in \cref{thm:tradeoff}. 

Throughout, we assume the relation between $y_i$ and $\mathbf{x}_i=(x_{i1},x_{i2})$ follows a bivariate linear model:
\begin{equation*}
    y_i = g(\mathbf{x}_i)+ \epsilon_i = \beta_0 +\beta_1x_{i1}+\beta_2x_{i2} +\epsilon_i
\end{equation*}
where $\beta = (\beta_0, \beta_1, \beta_2)$ is fixed, $\epsilon_i$ has mean $0$ and variance $\sigma^2$, $\epsilon_i \perp\!\!\!\perp \mathbf{x}_i$. We assume Gaussianity on $\mathbf{x}_i\overset{iid}{\sim}\mathcal{N}(\mathbf{0},\Sigma)$,  accordingly $\text{Var}(\mathbf{x}_i\beta) =\beta_1^2\Sigma_{11} +2\beta_1\beta_2\Sigma_{12}+\beta_2^2\Sigma_{22} := \phi(\Sigma)$. Let $s:=\frac{\text{Var}(g(\mathbf{x}_i))}{\sigma^2}=\frac{\phi(\Sigma)}{\sigma^2}$ be the signal-to-noise ratio (SNR), and keep $\phi(\Sigma)$ fixed. 

\vspace{3mm}
We investigate variance and covariance under the following two modeling assumptions: one with the correct model specifications, the other with wrong model specifications and only including one of the variables in regression.

\subsection{Correct Model Specifications}
First, we assume the regression model is correctly specified:
\begin{equation*}
    \begin{aligned}
    f_j(\mathbf{x}_i) &= \beta_0 +\beta_1x_{i1}+\beta_2x_{i2},\quad \forall 1\leq j\leq J\\
    \hat{f}_j(\mathbf{x}^*) &= \mathbf{x}^*(\mathbf{X}^T\mathbf{X})^{-1}\mathbf{X}^Ty
    \end{aligned}
\end{equation*}
where $\mathbf{X}\in\mathbb{R}^{N_j\times3}$ is a bootstrap feature matrix with a column of $\mathbf{1}$, $N_j$ is the effective number of unique samples, $(\mathbf{X}^T\mathbf{X})^{-1}\mathbf{X}^Ty$ is the OLS estimator, $\mathbf{x}^* \perp\!\!\!\perp \mathbf{X}$ is a new independent data point.

\vspace{3mm}
We are interested in the following two statistics:
\begin{equation*}
    \begin{aligned}
    \psi(s)&:=\mathbb{E}[\tau^2(\mathbf{x}^*, s)] = \mathbb{E}\left[\textup{Var}(\hat{f}_j(\mathbf{x}^*)|\mathbf{x}^*)\right], \quad1\leq j\leq J\\
    \omega(s)&:=\mathbb{E}[\rho(\mathbf{x}^*, s)] = \mathbb{E}\left[\textup{Cov}(\hat{f}_j(\mathbf{x}^*), \hat{f}_k(\mathbf{x}^*)|\mathbf{x}^*)\right], \quad 1\leq j<k\leq J
    \end{aligned}
\end{equation*}

They can be computed as follows:
\begin{equation*}
    \begin{aligned}
    \mathbb{E}\left[\textup{Var}(\hat{f}_j(\mathbf{x}^*)|\mathbf{x}^*)\right]
    &=\mathbb{E}\left[\textup{Var}(\mathbf{x}^*(\mathbf{X}^T\mathbf{X})^{-1}\mathbf{X}^Ty|\mathbf{x}^*)\right]\\
    &=\mathbb{E}\left[\textup{Var}(\mathbf{x}^*(\mathbf{X}^T\mathbf{X})^{-1}\mathbf{X}^T(\mathbf{X}\beta +\epsilon)|\mathbf{x}^*)\right]\\
    &=\mathbb{E}\left[\mathbf{x}^*\mathbb{E}[(\mathbf{X}^T\mathbf{X})^{-1}\mathbf{X}^T\epsilon\epsilon^T\mathbf{X}(\mathbf{X}^T\mathbf{X})^{-1}]\mathbf{x}^{*T}\right]\\
    &=\sigma^2\text{Tr}\left(\mathbb{E}\left[\mathbf{x}^*\mathbb{E}[(\mathbf{X}^T\mathbf{X})^{-1}]\mathbf{x}^{*T}\right]\right)\\
    &=\frac{\sigma^2}{N_j-3} =O\left(\frac{1}{s}\right)
    \end{aligned}
\end{equation*}
where the second last equality holds by moments of Wishart distribution and inverse-Wishart distribution.

\begin{equation*}
    \begin{aligned}
    \mathbb{E}\left[\textup{Cov}(\hat{f}_j(\mathbf{x}^*), \hat{f}_k(\mathbf{x}^*)|\mathbf{x}^*)\right] &=\mathbb{E}\left[\textup{Cov}\left(\mathbf{x}^*(\mathbf{X}^{(j)T}\mathbf{X}^{(j)})^{-1}\mathbf{X}^{(j)T}y^{(j)}, \mathbf{x}^*(\mathbf{X}^{(k)T}\mathbf{X}^{(k)})^{-1}\mathbf{X}^{(k)T}y^{(k)}|\mathbf{x}^*\right)\right]\\
    &=\mathbb{E}\left[\mathbf{x}^*\textup{Cov}\left((\mathbf{X}^{(j)T}\mathbf{X}^{(j)})^{-1}\mathbf{X}^{(j)T}\epsilon^{(j)}, (\mathbf{X}^{(k)T}\mathbf{X}^{(k)})^{-1}\mathbf{X}^{(k)T}\epsilon^{(k)}\right)\mathbf{x}^{*T}\right]\\
    &=\text{Tr}\left(\mathbb{E}\left[\mathbf{x}^*\mathbb{E}\left[(\mathbf{X}^{(j)T}\mathbf{X}^{(j)})^{-1}\mathbf{X}^{(j)T}\epsilon^{(j)}\epsilon^{(k)T}\mathbf{X}^{(k)}(\mathbf{X}^{(k)T}\mathbf{X}^{(k)})^{-1}\right]\mathbf{x}^{*T}\right]\right)\\
    &=\frac{\sigma^2}{N_{jk}-3} =O\left(\frac{1}{s}\right)
    \end{aligned}
\end{equation*}
where the second last equality uses the fact that $\mathbb{E}[\epsilon^{(j)}\epsilon^{(k)T}]$ has most of its entries equal $0$, except for when two rows from $\mathbf{X}^{(j)}$ and $\mathbf{X}^{(k)}$ are identical, $N_{jk}$ is the effective number of identical samples.

\vspace{3mm}
In summary, we showed that under linear assumptions and correct model specifications, both variance and covariance of the base learners $\psi(s), \omega(s) = O\left(\frac{1}{s}\right)$.

\subsection{Incorrect Model Specifications}
Now, we assume the regression model is incorrectly specified:
\begin{equation*}
    \begin{aligned}
    f_j(\mathbf{x}_i) &= \beta_0 +\beta_1x_{i1},\quad \forall 1\leq j\leq J\\
    \hat{f}_j(\mathbf{x}^*) &= \mathbf{x}^*(\mathbf{X_1}^T\mathbf{X_1})^{-1}\mathbf{X_1}^Ty
    \end{aligned}
\end{equation*}
where $\mathbf{X_1}\in\mathbb{R}^{N_j\times2}$ is a bootstrap feature matrix with the intercept and only $\{x_{i1}\}_{i=1}^{N_j}$. We further define $\mathbf{X_2}\in\mathbb{R}^{N_j\times1}$ contains $\{x_{i2}\}_{i=1}^{N_j}$, partial coefficients $\beta_{01} = (\beta_0, \beta_1)$.

\vspace{3mm}
Now variance and covariance of the univariate base learners can are as follows:
\begin{equation*}
    \begin{aligned}
    \mathbb{E}\left[\textup{Var}(\hat{f}_j(\mathbf{x}^*)|\mathbf{x}^*)\right]
    &=\mathbb{E}\left[\textup{Var}(\mathbf{x}^*(\mathbf{X_1}^T\mathbf{X_1})^{-1}\mathbf{X_1}^Ty|\mathbf{x}^*)\right]\\
    &=\mathbb{E}\left[\textup{Var}(\mathbf{x}^*(\mathbf{X_1}^T\mathbf{X_1})^{-1}\mathbf{X_1}^T(\mathbf{X_1}\beta_{01} + \mathbf{X_2}\beta_2+\epsilon)|\mathbf{x}^*)\right]\\
    &=\beta_2^2\text{Tr}\left(\mathbb{E}\left[\mathbf{x}^*\textup{Var}((\mathbf{X_1}^T\mathbf{X_1})^{-1}\mathbf{X_1}^T\mathbf{X_2})\mathbf{x}^{*T}\right]\right)+\sigma^2\text{Tr}\left(\mathbb{E}\left[\mathbf{x}^*\mathbb{E}[(\mathbf{X_1}^T\mathbf{X_1})^{-1}]\mathbf{x}^{*T}\right]\right)\\
    &=\beta_2^2\text{Tr}\left(\Sigma\textup{Var}((\mathbf{X_1}^T\mathbf{X_1})^{-1}\mathbf{X_1}^T\mathbf{X_2})\right)+\frac{\sigma^2}{N_j-2}
    \end{aligned}
\end{equation*}

If we assume further that $\mathbf{X_1}$ is independent of $\mathbf{X_2}$, then the above can be further simplified as:
\begin{equation*}
    \begin{aligned}
    \textup{Var}\left((\mathbf{X_1}^T\mathbf{X_1})^{-1}\mathbf{X_1}^T\mathbf{X_2}\right)
    =&\mathbb{E}\left[(\mathbf{X_1}^T\mathbf{X_1})^{-1}\mathbf{X_1}^T\mathbf{X_2}\mathbf{X_2}^T\mathbf{X_1}(\mathbf{X_1}^T\mathbf{X_1})^{-1}\right]\\
    &+\mathbb{E}\left[(\mathbf{X_1}^T\mathbf{X_1})^{-1}\mathbf{X_1}^T\mathbf{X_2}\right]^T\mathbb{E}\left[(\mathbf{X_1}^T\mathbf{X_1})^{-1}\mathbf{X_1}^T\mathbf{X_2}\right]\\
    =&\mathbb{E}\left[\mathbf{X_2}\mathbf{X_2}^T\right]\mathbb{E}\left[(\mathbf{X_1}^T\mathbf{X_1})^{-1}\right]
    \end{aligned}
\end{equation*}
where the second equality holds as $\mathbf{X_2}$ has independent rows, and $\mathbb{E}\left[\mathbf{X_2}\mathbf{X_2}^T\right]$ is diagonal with identical value. Accordingly, 
\begin{equation*}
    \mathbb{E}\left[\textup{Var}(\hat{f}_j(\mathbf{x}^*)|\mathbf{x}^*)\right] =\frac{\beta_2^2\Sigma_{22}}{N_j-2} +\frac{\sigma^2}{N_j-2} = O(s)+O\left(\frac{1}{s}\right)= O(s)
\end{equation*}
which shows that under-parametrized base learners with incorrect model specifications do not benefit from variance reduction as SNR grows. Rather, they suffer from variance increase, as they are missing more information about the response.

\begin{equation*}
    \begin{aligned}
    \mathbb{E}\left[\textup{Cov}(\hat{f}_j(\mathbf{x}^*), \hat{f}_k(\mathbf{x}^*)|\mathbf{x}^*)\right] &=\mathbb{E}\left[\textup{Cov}\left(\mathbf{x}^*(\mathbf{X_1}^{(j)T}\mathbf{X_1}^{(j)})^{-1}\mathbf{X_1}^{(j)T}y^{(j)}, \mathbf{x}^*(\mathbf{X_1}^{(k)T}\mathbf{X_1}^{(k)})^{-1}\mathbf{X_1}^{(k)T}y^{(k)}|\mathbf{x}^*\right)\right]\\
    =&\beta_2^2\mathbb{E}\left[\mathbf{x}^*\textup{Cov}\left((\mathbf{X_1}^{(j)T}\mathbf{X_1}^{(j)})^{-1}\mathbf{X_1}^{(j)T}\mathbf{X_2}^{(j)}, (\mathbf{X_1}^{(k)T}\mathbf{X_1}^{(k)})^{-1}\mathbf{X_1}^{(k)T}\mathbf{X_2}^{(k)}\right)\mathbf{x}^{*T}\right]\\
    &+\mathbb{E}\left[\mathbf{x}^*\textup{Cov}\left((\mathbf{X_1}^{(j)T}\mathbf{X_1}^{(j)})^{-1}\mathbf{X_1}^{(j)T}\epsilon^{(j)}, (\mathbf{X_1}^{(k)T}\mathbf{X_1}^{(k)})^{-1}\mathbf{X_1}^{(k)T}\epsilon^{(k)}\right)\mathbf{x}^{*T}\right]\\
    =&\frac{\beta_2^2\Sigma_{22}}{N_{jk}-2}+\frac{\sigma^2}{N_{jk}-2}\\
    =&O(s)+O\left(\frac{1}{s}\right)
    \end{aligned}
\end{equation*}

\vspace{3mm}
In summary, we showed that under linear assumptions and incorrect model specifications, for under-parametrized linear regression base learners $\psi(s), \omega(s) = O\left(s\right)$. In other words, both variance and covariance of the under-parametrized base learners scale up as SNR grows.

%% file: appendix/CCL.tex
As mentioned in \cref{subsec:CCL}, we run our analysis on three independent cell line screen datasets, including \textit{the Genentech Cell line Screening Initiative} (gCSI), \textit{the Cancer Therapeutics Response Portal} (CTRPv2), and \textit{the Genomics of Drug Sensitivity in Cancer} (GDSC2).

To establish a benchmark for performance assessment, we make use of the overlapping cell lines across datasets, as an evaluation of how noisy the ground truth is. The idea is as follows: suppose we have $k$ cell lines in common between training set CTRPv2 and validation set gCSI, with response values $z_1, ..., z_k$ in CTPRv2 and $y_1,...,y_k$ in gCSI. Then we compare $\text{MSE}(z,y)$ to $\text{MSE}(\hat{y},y)$, where $\hat{y}$ are predictions generated by the model. Note that here we no longer center and scale the response $y$ prior to fitting the model, so that both of the mean squared errors are on the original scale and comparable.

\autoref{fig:gCSI} and \autoref{fig:GDSC2} show the comparison of the validation error $\text{MSE}(\hat{y},y)$ from different methods, with the noise $\text{MSE}(z,y)$ in ground truth. We can observe that all three methods are reaching the \textit{performance upper bound}, i.e. MSE are comparable with the noise in ground truth. Furthermore, Lasso forest almost always achieves the optimal performance among all three methods on both validation sets.
\begin{figure}[H]
    \centering
    \includegraphics[width=1\linewidth]{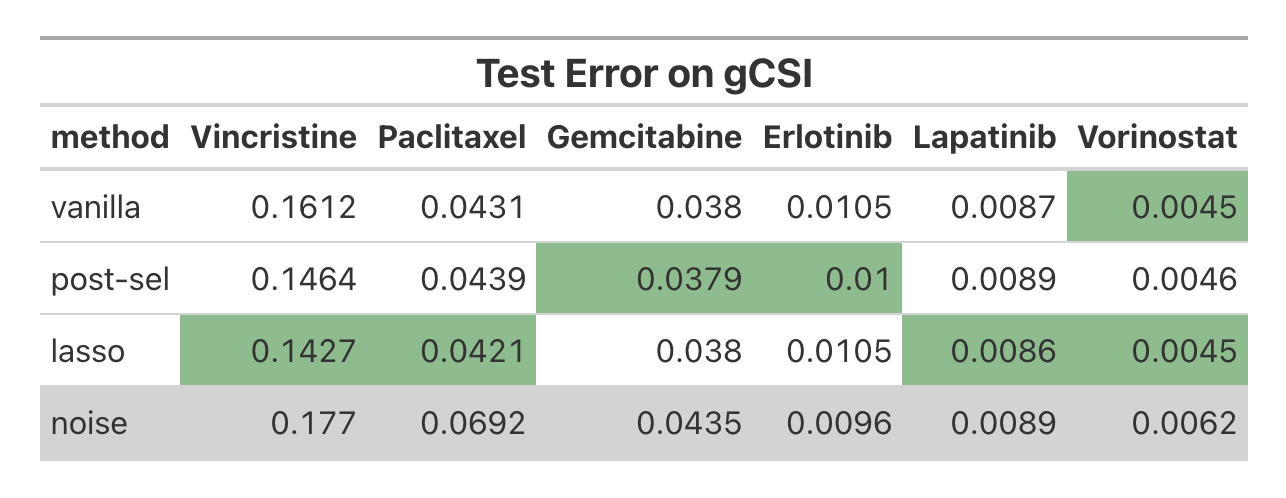}
    \caption{\em{Test Error on gCSI, Trained on CTRPv2}}
    \label{fig:gCSI}
\end{figure}

\begin{figure}[H]
    \centering
    \includegraphics[width=1\linewidth]{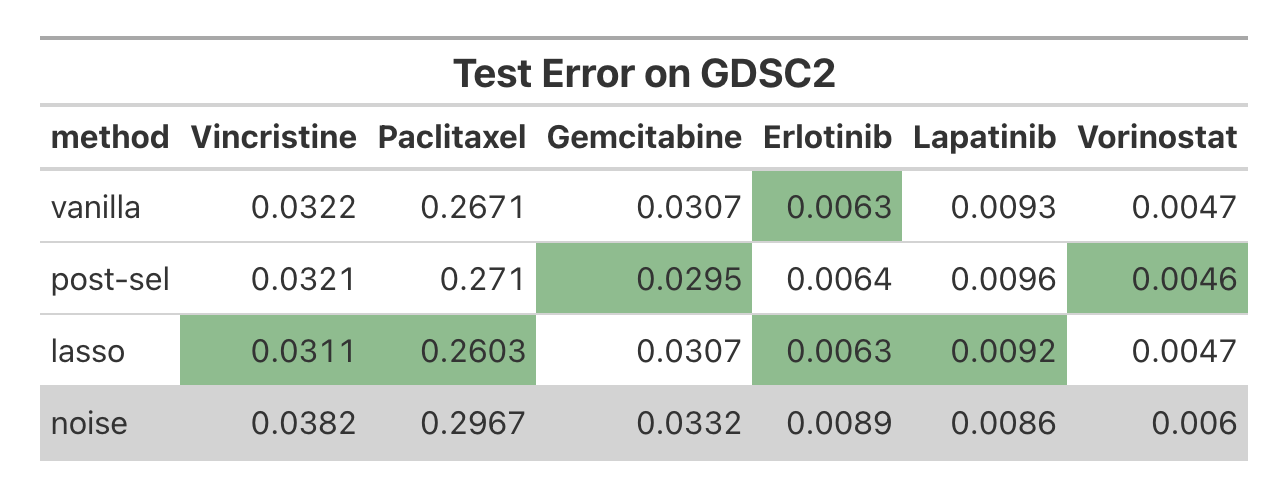}
    \caption{\em{Test Error on GDSC2, Trained on CTRPv2}}
    \label{fig:GDSC2}
\end{figure}

%% file: reference.bib
@article{wang2021improving,
  title={Improving random forest algorithm by Lasso method},
  author={Wang, Hui and Wang, Guizhi},
  journal={Journal of Statistical Computation and Simulation},
  volume={91},
  number={2},
  pages={353--367},
  year={2021},
  publisher={Taylor \& Francis}
}

@article{beck2024hedged,
  title={The Hedged Random Forest},
  author={Beck, Elliot and Kozbur, Damian and Wolf, Michael},
  journal={Available at SSRN 5032102},
  year={2024}
}

@article{subramanian2017next,
  title={A next generation connectivity map: L1000 platform and the first 1,000,000 profiles},
  author={Subramanian, Aravind and Narayan, Rajiv and Corsello, Steven M and Peck, David D and Natoli, Ted E and Lu, Xiaodong and Gould, Joshua and Davis, John F and Tubelli, Andrew A and Asiedu, Jacob K and others},
  journal={Cell},
  volume={171},
  number={6},
  pages={1437--1452},
  year={2017},
  publisher={Elsevier}
}

@article{winham2013weighted,
  title={A weighted random forests approach to improve predictive performance},
  author={Winham, Stacey J and Freimuth, Robert R and Biernacka, Joanna M},
  journal={Statistical Analysis and Data Mining: The ASA Data Science Journal},
  volume={6},
  number={6},
  pages={496--505},
  year={2013},
  publisher={Wiley Online Library}
}

@article{tibshirani1996regression,
  title={Regression shrinkage and selection via the lasso},
  author={Tibshirani, Robert},
  journal={Journal of the Royal Statistical Society Series B: Statistical Methodology},
  volume={58},
  number={1},
  pages={267--288},
  year={1996},
  publisher={Oxford University Press}
}

@inproceedings{ho1995random,
  title={Random decision forests},
  author={Ho, Tin Kam},
  booktitle={Proceedings of 3rd international conference on document analysis and recognition},
  volume={1},
  pages={278--282},
  year={1995},
  organization={IEEE}
}

@article{breiman2001random,
  title={Random forests},
  author={Breiman, Leo},
  journal={Machine learning},
  volume={45},
  pages={5--32},
  year={2001},
  publisher={Springer}
}

@article{friedman2008predictive,
  title={Predictive learning via rule ensembles},
  author={Friedman, Jerome H and Popescu, Bogdan E},
  journal={The Annals of Applied Statistics},
  volume = {2},
  number = {3},
  publisher = {Institute of Mathematical Statistics},
  pages = {916 -- 954},
  year={2008}
}

@inproceedings{robnik2004improving,
  title={Improving random forests},
  author={Robnik-{\v{S}}ikonja, Marko},
  booktitle={European conference on machine learning},
  pages={359--370},
  year={2004},
  organization={Springer}
}

@inproceedings{bernard2009selection,
  title={On the selection of decision trees in random forests},
  author={Bernard, Simon and Heutte, Laurent and Adam, Sebastien},
  booktitle={2009 International joint conference on neural networks},
  pages={302--307},
  year={2009},
  organization={IEEE}
}

@article{paul2018improved,
  title={Improved random forest for classification},
  author={Paul, Angshuman and Mukherjee, Dipti Prasad and Das, Prasun and Gangopadhyay, Abhinandan and Chintha, Appa Rao and Kundu, Saurabh},
  journal={IEEE Transactions on Image Processing},
  volume={27},
  number={8},
  pages={4012--4024},
  year={2018},
  publisher={IEEE}
}

@article{huang2010variable,
  title={Variable selection in nonparametric additive models},
  author={Huang, Jian and Horowitz, Joel L and Wei, Fengrong},
  journal={The Annals of Statistics},
  volume={38},
  number={4},
  pages={2282},
  year={2010}
}

@article{chouldechova2015generalized,
  title={Generalized additive model selection},
  author={Chouldechova, Alexandra and Hastie, Trevor},
  journal={arXiv preprint arXiv:1506.03850},
  year={2015}
}

@article{athey2019generalized,
  title={Generalized random forests},
  author={Athey, Susan and Tibshirani, Julie and Wager, Stefan},
 volume = {47},
 journal = {The Annals of Statistics},
 number = {2},
 publisher = {Institute of Mathematical Statistics},
 pages = {1148 -- 1178},
  year={2019}
}

@article{pace1997sparse,
  title     = {Sparse Spatial Autoregressions},
  author    = {Pace, R. Kelley and Barry, Ronald},
  journal   = {Statistics and Probability Letters},
  volume    = {33},
  number    = {3},
  pages     = {291--297},
  year      = {1997},
  publisher = {Elsevier}
}

@misc{hopkins1999spambase,
  author       = {Hopkins, Mark and Reeber, Erik and Forman, George and Suermondt, Jaap},
  title        = {{Spambase}},
  year         = {1999},
  howpublished = {UCI Machine Learning Repository},
  note         = {{DOI}: https://doi.org/10.24432/C53G6X}
}

@article{freund2001adaptive,
  title={An Adaptive Version of the Boost by Majority Algorithm},
  author={Freund, Yoav},
  journal={Machine Learning},
  volume={43},
  number={3},
  pages={293},
  year={2001},
  publisher={Springer Nature BV}
}

@article{hastie2009random,
  title={Random forests},
  author={Hastie, Trevor and Tibshirani, Robert and Friedman, Jerome and Hastie, Trevor and Tibshirani, Robert and Friedman, Jerome},
  journal={The elements of statistical learning: Data mining, inference, and prediction},
  pages={587--604},
  year={2009},
  publisher={Springer}
}

@article{box1988signal,
  title={Signal-to-noise ratios, performance criteria, and transformations},
  author={Box, George},
  journal={Technometrics},
  volume={30},
  number={1},
  pages={1--17},
  year={1988},
  publisher={Taylor \& Francis}
}

@ARTICLE{4839045,
  author={Wainwright, Martin J.},
  journal={IEEE Transactions on Information Theory}, 
  title={Sharp Thresholds for High-Dimensional and Noisy Sparsity Recovery Using $\ell _{1}$ -Constrained Quadratic Programming (Lasso)}, 
  year={2009},
  volume={55},
  number={5},
  pages={2183-2202},
  doi={10.1109/TIT.2009.2016018}}

@article{10.1214/07-AOS582,
author = {Nicolai Meinshausen and Bin Yu},
title = {{Lasso-type recovery of sparse representations for high-dimensional data}},
volume = {37},
journal = {The Annals of Statistics},
number = {1},
publisher = {Institute of Mathematical Statistics},
pages = {246 -- 270},
year = {2009},
doi = {10.1214/07-AOS582}
}

@article {Nguyen2023.01.08.522775,
	author = {Nguyen, Julia and Mammoliti, Anthony and Nair, Sisira Kadambat and So, Emily and Abbas-Aghababazadeh, Farnoosh and Eeles, Christoper and Smith, Ian and Smirnov, Petr and He, Housheng Hansen and Tsao, Ming-Sound and Haibe-Kains, Benjamin},
	title = {Detection of circular RNAs and their potential as biomarkers predictive of drug response},
	elocation-id = {2023.01.08.522775},
	year = {2023},
	doi = {10.1101/2023.01.08.522775},
	publisher = {Cold Spring Harbor Laboratory},
	eprint = {https://www.biorxiv.org/content/early/2023/01/08/2023.01.08.522775.full.pdf},
	journal = {bioRxiv}
}

@article{haverty2016reproducible,
  title={Reproducible pharmacogenomic profiling of cancer cell line panels},
  author={Haverty, Peter M and Lin, Eva and Tan, Jenille and Yu, Yihong and Lam, Billy and Lianoglou, Steve and Neve, Richard M and Martin, Scott and Settleman, Jeff and Yauch, Robert L and others},
  journal={Nature},
  volume={533},
  number={7603},
  pages={333--337},
  year={2016},
  publisher={Nature Publishing Group UK London}
}

@article{klijn2015comprehensive,
  title={A comprehensive transcriptional portrait of human cancer cell lines},
  author={Klijn, Christiaan and Durinck, Steffen and Stawiski, Eric W and Haverty, Peter M and Jiang, Zhaoshi and Liu, Hanbin and Degenhardt, Jeremiah and Mayba, Oleg and Gnad, Florian and Liu, Jinfeng and others},
  journal={Nature biotechnology},
  volume={33},
  number={3},
  pages={306--312},
  year={2015},
  publisher={Nature Publishing Group}
}

@article{rees2016correlating,
  title={Correlating chemical sensitivity and basal gene expression reveals mechanism of action},
  author={Rees, Matthew G and Seashore-Ludlow, Brinton and Cheah, Jaime H and Adams, Drew J and Price, Edmund V and Gill, Shubhroz and Javaid, Sarah and Coletti, Matthew E and Jones, Victor L and Bodycombe, Nicole E and others},
  journal={Nature chemical biology},
  volume={12},
  number={2},
  pages={109--116},
  year={2016},
  publisher={Nature Publishing Group US New York}
}

@article{seashore2015harnessing,
  title={Harnessing connectivity in a large-scale small-molecule sensitivity dataset},
  author={Seashore-Ludlow, Brinton and Rees, Matthew G and Cheah, Jaime H and Cokol, Murat and Price, Edmund V and Coletti, Matthew E and Jones, Victor and Bodycombe, Nicole E and Soule, Christian K and Gould, Joshua and others},
  journal={Cancer discovery},
  volume={5},
  number={11},
  pages={1210--1223},
  year={2015},
  publisher={American Association for Cancer Research}
}

@article{basu2013interactive,
  title={An interactive resource to identify cancer genetic and lineage dependencies targeted by small molecules},
  author={Basu, Amrita and Bodycombe, Nicole E and Cheah, Jaime H and Price, Edmund V and Liu, Ke and Schaefer, Giannina I and Ebright, Richard Y and Stewart, Michelle L and Ito, Daisuke and Wang, Stephanie and others},
  journal={Cell},
  volume={154},
  number={5},
  pages={1151--1161},
  year={2013},
  publisher={Elsevier}
}

@article{yang2012genomics,
  title={Genomics of Drug Sensitivity in Cancer (GDSC): a resource for therapeutic biomarker discovery in cancer cells},
  author={Yang, Wanjuan and Soares, Jorge and Greninger, Patricia and Edelman, Elena J and Lightfoot, Howard and Forbes, Simon and Bindal, Nidhi and Beare, Dave and Smith, James A and Thompson, I Richard and others},
  journal={Nucleic acids research},
  volume={41},
  number={D1},
  pages={D955--D961},
  year={2012},
  publisher={Oxford University Press}
}

@article{iorio2016landscape,
  title={A landscape of pharmacogenomic interactions in cancer},
  author={Iorio, Francesco and Knijnenburg, Theo A and Vis, Daniel J and Bignell, Graham R and Menden, Michael P and Schubert, Michael and Aben, Nanne and Gon{\c{c}}alves, Emanuel and Barthorpe, Syd and Lightfoot, Howard and others},
  journal={Cell},
  volume={166},
  number={3},
  pages={740--754},
  year={2016},
  publisher={Elsevier}
}

@article{garnett2012systematic,
  title={Systematic identification of genomic markers of drug sensitivity in cancer cells},
  author={Garnett, Mathew J and Edelman, Elena J and Heidorn, Sonja J and Greenman, Chris D and Dastur, Anahita and Lau, King Wai and Greninger, Patricia and Thompson, I Richard and Luo, Xi and Soares, Jorge and others},
  journal={Nature},
  volume={483},
  number={7391},
  pages={570--575},
  year={2012},
  publisher={Nature Publishing Group UK London}
}

@article{mammoliti2021orchestrating,
  title={Orchestrating and sharing large multimodal data for transparent and reproducible research},
  author={Mammoliti, Anthony and Smirnov, Petr and Nakano, Minoru and Safikhani, Zhaleh and Eeles, Christopher and Seo, Heewon and Nair, Sisira Kadambat and Mer, Arvind S and Smith, Ian and Ho, Chantal and others},
  journal={Nature communications},
  volume={12},
  number={1},
  pages={5797},
  year={2021},
  publisher={Nature Publishing Group UK London}
}

@article{liu2019integrative,
  title={Integrative molecular and clinical modeling of clinical outcomes to PD1 blockade in patients with metastatic melanoma},
  author={Liu, David and Schilling, Bastian and Liu, Derek and Sucker, Antje and Livingstone, Elisabeth and Jerby-Arnon, Livnat and Zimmer, Lisa and Gutzmer, Ralf and Satzger, Imke and Loquai, Carmen and others},
  journal={Nature medicine},
  volume={25},
  number={12},
  pages={1916--1927},
  year={2019},
  publisher={Nature Publishing Group US New York}
}

@article{harrell1982evaluating,
  title={Evaluating the yield of medical tests},
  author={Harrell, Frank E and Califf, Robert M and Pryor, David B and Lee, Kerry L and Rosati, Robert A},
  journal={Jama},
  volume={247},
  number={18},
  pages={2543--2546},
  year={1982},
  publisher={American Medical Association}
}

@article{rhee2003human,
  title={Human immunodeficiency virus reverse transcriptase and protease sequence database},
  author={Rhee, Soo-Yon and Gonzales, Matthew J and Kantor, Rami and Betts, Bradley J and Ravela, Jaideep and Shafer, Robert W},
  journal={Nucleic acids research},
  volume={31},
  number={1},
  pages={298--303},
  year={2003},
  publisher={Oxford University Press}
}

@article{shafer2006rationale,
  title={Rationale and uses of a public HIV drug-resistance database},
  author={Shafer, Robert W},
  journal={The Journal of infectious diseases},
  volume={194},
  number={Supplement\_1},
  pages={S51--S58},
  year={2006},
  publisher={University of Chicago Press}
}

@article{mallinar2024minimum,
  title={Minimum-norm interpolation under covariate shift},
  author={Mallinar, Neil and Zane, Austin and Frei, Spencer and Yu, Bin},
  journal={arXiv preprint arXiv:2404.00522},
  year={2024}
}

@article{hastie2022surprises,
  title={Surprises in high-dimensional ridgeless least squares interpolation},
  author={Hastie, Trevor and Montanari, Andrea and Rosset, Saharon and Tibshirani, Ryan J},
  journal={The Annals of Statistics},
  volume={50},
  number={2},
  pages={949},
  year={2022}
}

@article{wright2017ranger,
  title={ranger: A fast implementation of random forests for high dimensional data in C++ and R},
  author={Wright, Marvin N and Ziegler, Andreas},
  journal={Journal of statistical software},
  volume={77},
  pages={1--17},
  year={2017}
}
